\documentclass[11pt]{article}
\pdfoutput=1
\usepackage{booktabs}
\usepackage{pgfplots}
\usepackage{graphicx}
\usepackage{tikz}
\usepackage{amsmath, amsthm}
\usepackage{amssymb}
\usepackage{graphicx}
\usepackage{setspace}
\usepackage{xfrac}
\usepackage{hyperref}
\usepackage{xstring}
\usepackage{xspace}
\usepackage{bm}
\usepackage{microtype}
\usepackage{graphicx}

\usepackage{booktabs} 

\usepackage{comment}
\pgfplotsset{compat=1.14}

\usepackage{gensymb}
\usepackage{hyperref}       
\usepackage{amsfonts}       
\usepackage{nicefrac}       
\usepackage{microtype}      
\usepackage{amsmath}
\usepackage{amsfonts}
\usepackage{amssymb}
\usepackage{comment}
\usepackage[ruled,vlined]{algorithm2e}

\usepackage{subcaption}
\usepackage{hyperref}

\usepackage{amsthm}
\usepackage{amscd}
\usepackage{t1enc}
\usepackage{enumerate}
\usepackage{enumitem}
\usepackage[mathscr]{eucal}
\usepackage{indentfirst}
\usepackage{listings}
\usepackage{graphicx}
\usepackage{graphics}
\usepackage{pict2e}
\usepackage{epic}
\usepackage{epstopdf} 
\usepackage{wrapfig}

\newenvironment{sproof}{%
  \proof}{\endproof}

\newcommand{\cH}{\mathcal{H}}

\newcommand{\cR}{\mathcal{R}}

\newcommand{\bE}{\mathbb{E}}
\newcommand{\bP}{\mathbb{P}}
\newcommand{\bI}{\mathbb{I}}

\renewcommand{\tilde}{\widetilde}
\renewcommand{\nu}{\vartheta}

\usepackage{float}

\allowdisplaybreaks
\usepackage{fullpage}
\newtheorem{theorem}{Theorem}
\newtheorem*{theorem*}{Theorem}

\newtheorem{proposition}{Proposition}
\newtheorem*{proposition*}{Proposition}
\theoremstyle{definition}

\newtheorem{definition}{Definition}

\newtheorem*{remark*}{Remark}
\date{}
\title{\textbf{Consistent Estimators for Learning to Defer to an Expert}}
\author{Hussein Mozannar \thanks{Massachusetts Institute of Technology. Email: \texttt{mozannar@mit.edu}} \and David Sontag \thanks{Massachusetts Institute of Technology. Email: \texttt{dsontag@csail.mit.edu}}}
\begin{document}
\maketitle
\begin{abstract}
Learning algorithms are often used in conjunction with expert decision makers in practical scenarios, however this fact is largely ignored when designing these algorithms. In this paper we explore how to learn predictors that can either predict or choose to defer the decision to a downstream expert. Given only samples of the expert's decisions, we give a procedure based on learning a classifier and a rejector and analyze it theoretically. Our approach is based on a novel reduction to cost sensitive learning where we give a  consistent surrogate loss for cost sensitive learning that generalizes the cross entropy loss. We show the effectiveness of our approach on a variety of experimental tasks. 
\end{abstract}

\section{Introduction}
Machine learning systems are now being deployed in  settings to complement human decision makers such as in healthcare \cite{hamid2017machine,raghu2019algorithmic}, risk assessment \cite{green2019disparate} and content moderation \cite{link2016human}. These models are either used as a tool to help the downstream human decision maker: judges relying on algorithmic risk assessment tools  \cite{green2019principles} and risk scores being used in the ICU \cite{sepsistrial}, or instead these learning models are solely used to make the final prediction on a selected subset of examples \cite{madras2018predict,raghu2019algorithmic}. A current application of the latter setting is Facebook's and  other online platforms content moderation approach \cite{fbmoderate,jhaver2019human}: an algorithm is used to filter easily detectible inappropriate content and the rest of the examples are screened by a team of human moderators.
Another motivating application arises in health care settings, for example deep neural networks can outperform radiologists in detecting pneumonia from chest X-rays \cite{irvin2019chexpert}, however, many obstacles are limiting complete automation, an intermediate step to automating this task will be the use of  models as triage tools to complement radiologist expertise. 
Our focus in this work is to give theoretically sound approaches for machine learning models that can either predict or defer the decision to a downstream expert to complement and augment their capabilities. 

The learned model should adapt to the underlying human expert in order to achieve better performance than deploying the model or expert individually. In situations where we have limited data or model capacity, the gains from allowing the model to focus on regions where the expert is less accurate are expected to be more significant. However, even when data or model capacity are not concerns, the expert may have access to side-information unavailable to the learner due to  privacy concerns for example, the hard task is then to identify when we should defer without having access to this side-information. We will only assume in this work that we are allowed access to samples of the experts decisions or to costs of deferring, we believe that this is a reasonable assumption that can be achieved in practical settings.
Inspired by the literature on rejection learning \cite{cortes2016learning}, our approach will be to learn two functions: a classifier that can predict the target and a rejector which decides whether the classifier or the expert should predict.

We  start by formulating a natural loss function for the combined machine-expert system in section \ref{sec:problem}
and show a reduction from the expert deferral setting to cost sensitive learning. With this reduction in hand, we are able to give a novel convex surrogate loss that upper bounds our system loss and that is furthermore consistent in section \ref{sec:surrog}. This surrogate loss settles the open problem posed by \cite{ni2019possibility} for finding a consistent loss for multiclass rejection learning. Our proposed surrogate loss and approach requires only adding an additional output layer to existing model architectures and changing the loss function, hence it necessitates minimal to no added computational costs.  
In section \ref{sec:theory}, we show the limitations of approaches in the literature from a consistency point-of-view and then provide generalization bounds for minimizing the empirical loss.
To show the efficacy of our approach, we give experimental evidence on image classification datasets CIFAR-10 and CIFAR-100 using synthetic  and human experts based on \texttt{CIFAR10H} \cite{peterson2019human}, on a hate speech and offensive language detection task \cite{davidson2017automated}, and on classification of chest X-rays with synthetic experts in section \ref{sec:experiments}. To summarize, the contributions of this paper are the following:
\begin{itemize}
   \item We formalize the expert deferral setup and analyze it theoretically giving a generalization bound for solving the empirical problem.
   \item We propose a novel convex consistent surrogate loss $L_{CE}$ \eqref{eq:proposed_CE_loss} for expert deferral easily integrated into current learning pipelines.
   \item We provide a detailed experimental evaluation of our method and baselines from the literature on image and text classification tasks.
\end{itemize}

\section{Related Work}
Learning with a reject option, \emph{rejection learning}, has long been studied starting with \cite{chow1970optimum} who investigated the trade-off between accuracy and the rejection rate. The framework of rejection learning assumes a constant cost $c$ of deferring and hence the problem becomes to predict only if one is $1-c$ confident. Numerous works have proposed surrogate losses and uncertainty estimation methods to solve the problem \cite{bartlett2008classification, ramaswamy2018consistent,ni2019possibility, jiang2018trust}. \cite{cortes2016learning,cortes2016boosting} proposed a different approach by learning two functions: a classifier and a rejection function and analyzed the approach giving a kernel based algorithm in the binary setting. \cite{ni2019possibility} tried to extend their approach to the multiclass setting but failed to give a consistent surrogate loss and hence resorted to confidence based methods.

Recent work has started to explore models that defer to downstream experts, \cite{madras2018predict} considers an identical framework to the one considered here however their approach does not allow the model to adapt to the underlying expert and
the loss used is not consistent and requires an uncertainty estimate of the expert decisions. On the other hand, \cite{de2019regression} gives an approximate procedure to learn a linear model that picks a subset of the training data on which to defer and uses a nearest neighbor algorithm to defer on new examples, the approach used is only feasible for small dataset sizes and does not generalize beyond ridge regression. \cite{raghu2019algorithmic} considers binary classification with expert deferral, their approach is to learn a classifier ignoring the expert and obtain uncertainty estimates for both the expert and classifier and then defer based on which is higher, we detail the limitations of this approach in section \ref{sec:theory}. Concurrent work \cite{wilder2020learning} learns a model with the mixtures of expert loss first introduced in \cite{madras2018predict} and defers based on estimated model and expert confidence as in \cite{raghu2019algorithmic}.
Work on AI-assisted decision making has  focused on the reverse setting considered here: the expert chooses to accept or reject the decision of the classifier instead of  a learned rejector \cite{bansal2019updates,bansal2020optimizing}.
Additionally, the fairness in machine learning community has started to consider the fairness impact of having downstream decision makers \cite{madras2018predict,canetti2019soft,green2019disparate,dwork2018fairness} but in slightly different frameworks than the ones considered here and work has started to consider deferring in reinforcement learning \cite{meresht2020learning}.

A related framework to our setting is selective classification  \cite{el2010foundations} where instead of setting a cost for rejecting to predict one sets a constraint on the probability of rejection; here is no assumed downstream expert. Approaches range from deferring based on confidence scores \cite{geifman2017selective}, learning a deep network with two heads, one for predicting and the other for deferring \cite{geifman2019selectivenet} and learning with portfolio theory inspired loss functions \cite{ziyin2019deep}. Finally, our work bears resemblance to active learning with weak (the expert) and strong labelers (the ground truth) \cite{zhang2015active}.
\section{Problem Formulation}\label{sec:problem}

We are interested in predicting a target $Y \in \mathcal{Y}=\{1,\cdots,K\}$ based on covariates $X \in \mathcal{X}$ where $X,Y \sim \mathbf{P}$. We assume that we have query access to an expert $M$ that has access to a domain $\mathcal{Z}$  that may contain additional information than $\mathcal{X}$ to classify instances according to the target $\mathcal{Y}$. Querying the expert implies deferring the decision which incurs a cost $l_{exp}(x,y,m)$ that depends on the target $y$, covariate $x$ and the expert's prediction $m$. On the other hand, predicting without querying the expert implies that a classifier makes the final decision and incurs a cost $l(x,y,\hat{y})$ where $\hat{y}$ is the prediction of the classifier. 
 Our goal is to build a predictor $\hat{Y}: \mathcal{X} \to \mathcal{Y} \cup \{ \bot\}$ that can either predict or defer the decision to the expert denoted by $\bot$.  
 We can now formulate a natural system loss function $L$ for the system consisting of the classifier in conjunction with the expert:
\begin{align}\label{eq:system_loss_general}
 & L(\hat{Y}) =  \bE_{(x,y)\sim \mathbf{P},m \sim M|(x,y)} \ [ \ \underbrace{l(x,y,\hat{Y}(x))}_{\text{classifier cost}} \overbrace{\bI_{\hat{Y}(x)\neq \bot}}^{\text{predict}} +  \underbrace{l_{\textrm{exp}}(x,y,m)}_{\text{expert cost}} \overbrace{\bI_{\hat{Y}(x)= \bot}}^{\text{defer}} \  ]   
\end{align}

 Our strategy for learning the predictor $\hat{Y}$ will be to learn two separate functions $h: \mathcal{X} \to \mathcal{Y}$ (classifier)  and $r: \mathcal{X} \to \{0, 1 \}$ (rejector) and hence we write our loss as:
\begin{align}
    &L(h,r)= \label{eq:original_reject_loss}  \bE_{(x,y)\sim \mathbf{P},m \sim M|(x,y)} \ [ \ l(x,y,h(x)) \bI_{r(X) = 0} +  l_{\textrm{exp}}(x,y,m) \bI_{r(x)=1} \  ]  
\end{align}
\begin{figure}[h]  
\centering
  \includegraphics[clip,scale=0.6,trim={2cm 6.5cm 2cm 6cm}]{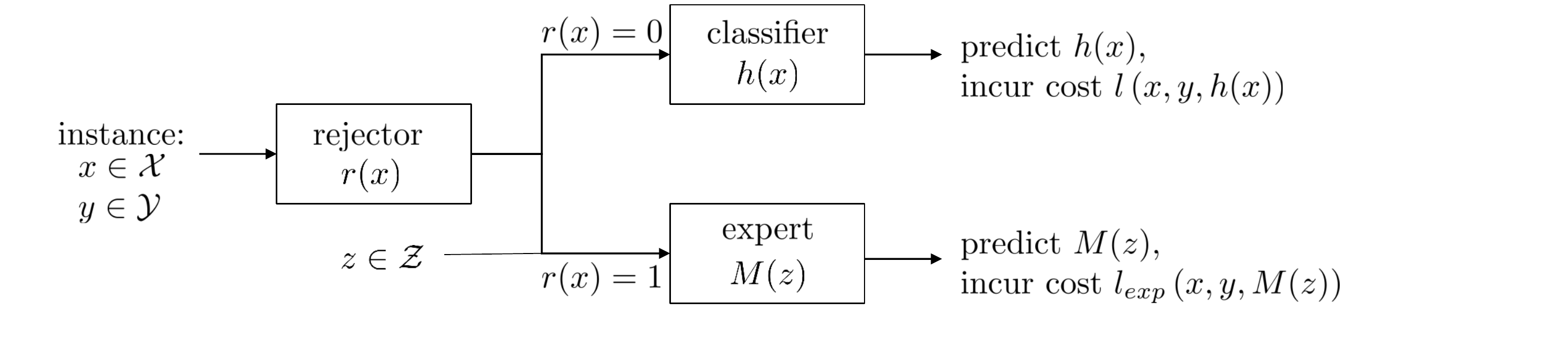}
    \caption{The expert deferral pipeline, the rejector first $r(x)$ decides who between the classifier $h(x)$ and expert $M(z)$  should predict and then whoever makes the final prediction incurs a specific cost.}
    \label{fig:setup}
\end{figure}
Figure \ref{fig:setup} illustrates our expert deferral setting with it's different components.
The above formulation is a generalization of the learning with rejection framework studied by \cite{cortes2016learning} as by setting $l_{\textrm{exp}}(x,y,m)=c$ for a constant $c>0$ the two objectives coincide. In \cite{madras2018predict}, the loss proposed assumes that the classifier and expert costs are the logistic loss between the target and their predictions in the binary target setting.

While our treatment extends to general forms of expert and classifier costs, we will pay particular attention in our theoretical analysis when the costs are the misclassification error with the target. Formally, we define a $0{-}1$ loss version of our system loss:
\begin{align}
    &L_{0{-}1}(h,r)= \label{eq:01_reject_loss}  \bE_{(x,y)\sim \mathbf{P},m \sim M|(x,y)} \ [ \ \bI_{h(x) \neq y} \bI_{r(x) = 0} +  \bI_{m \neq y}\bI_{r(x)=1} \  ]
\end{align}
One may also assume a constant additive cost function $c(x)$ for querying the expert depending on the instance $x$ making $l_{\textrm{exp}}(x,y,m)= \bI_{m \neq y} + c(x)$; such additive costs can be easily integrated into our analysis.

Our approach will be to cast this problem as a \emph{cost sensitive learning} problem over an augmented label space that includes the action of deferral. Let the random costs $\mathbf{c} \in \mathbb{R}_+^{K+1}$ where for $i \in [K]$,  $c(i)$ is the $i'$th component of $\mathbf{c}$ represents the cost of predicting $i \in \mathcal{Y}$ while $c[K+1]$ represents the cost of deferring to the expert. The goal of this setup is to learn a predictor $h: \mathcal{X} \to [K+1]$ minimizing the cost sensitive loss $\tilde{L}(h) :=\bE[c(h(x))]$. For example, giving an instance $(x,y)$, our loss \eqref{eq:original_reject_loss} is obtained by setting $c(i)=l(x,y,i)$ for $i \in [K]$ and $c(K+1)=l_{\textrm{exp}}(x,y,m)$.

For the majority of this paper we assume access to samples $S=\{(x_i,y_i,m_i)\}_{i=1}^n$ where $\{(x_i,y_i)\}_{i=1}^n$ are drawn i.i.d. from the unknown distribution $\mathbf{P}$ and $m_i$ is drawn from the distribution of the random variable $M|(X=x_i,Y=y_i)$ and access to the realizations of $l_{exp}$ and $l$ when required .


\section{Proposed Surrogate Loss}\label{sec:surrog}

It is clear that the system loss function \eqref{eq:original_reject_loss} is not only non-convex but also computationally hard to optimize. The usual approach in machine learning is to formulate upper bounding convex surrogate loss functions and optimize them in hopes of approximating the minimizers of the original loss \cite{bartlett2006convexity}. Work from rejection learning \cite{cortes2016learning,ni2019possibility} suggested learning two separate functions $h$ and $r$ and provided consistent convex surrogate loss functions only for the binary setting. We extend their proposed surrogates for our expert deferral setting for binary labels with slight modifications in  appendix \ref{apx:proofs}. Consistency is used to prove that a proposed surrogate loss is a good candidate and is often treated as a necessary condition. The issue with the proposed surrogates in \cite{cortes2016learning} for rejection learning is that when extended to the multiclass setting, it is impossible for them to be consistent as was shown by \cite{ni2019possibility}. Aside the consistency issue, \cite{ni2019possibility} found that simple baselines can outperform the proposed losses in practice.


The construction of our proposed surrogate loss for the multiclass expert deferral setting will be motivated via two  ways, the first is through  a novel reduction to cost sensitive learning and the second is inspired by the Bayes minimizer for the $0{-}1$ system loss \eqref{eq:01_reject_loss}.
Let $g_i: \mathcal{X} \to \mathbb{R}$ for $i \in [K+1]$ and define $h(x) = \arg \max_{i \in [K+1]}g_i$,
motivated by the success of the cross entropy loss, our proposed surrogate for cost-sensitive learning $\tilde{L}_{CE}$ takes the following form:
\begin{align}
&\tilde{L}_{CE}(g_1,\cdots,g_{K+1},x,c(1),\cdots,c(K+1)) =  - \sum_{i=1}^{K+1} (\max_{j \in [K+1]} c(j) - c(i)) \log\left( \frac{\exp(g_i(x))}{\sum_k \exp(g_k(x))} \right)
\end{align}
The loss $\tilde{L}_{CE}$ is a novel surrogate loss for cost sensitive learning that generalizes the cross entropy loss when the costs correspond to multiclass misclassification. The following proposition shows that the loss is consistent, meaning it's minimizer over all measurable functions agrees with the Bayes solution.

\begin{proposition}
 $\tilde{L}_{CE}$  is  convex in $\mathbf{g}$ and is a consistent loss function for $\tilde{L}$:
 \begin{center}
      let $\bm{\tilde{g}}=\arg  \inf_{\mathbf{g}} \bE\left[ \tilde{L}_{CE}(\mathbf{g},\mathbf{c}) |X=x\right]$, then:  $\arg  \max_{i \in [K+1]}\bm{\tilde{g}}_i = \arg \min_{i \in [K+1]} \bE[c(i)|X=x]$
 \end{center}
\label{prop:lce_cost_sensitive}
\end{proposition}
Proof of Proposition \ref{prop:lce_cost_sensitive} can be found in Appendix \ref{apx:proofs}; $\tilde{L}_{CE}$ is a simpler consistent alternative to the surrogates derived in \cite{chen2019surrogate} for cost sensitive learning.

Now we consider when the system loss function is $L_{0{-}1}$ \eqref{eq:01_reject_loss}, our approach is to treat deferral as a new class and construct a new label space $\mathcal{Y}^\bot=\mathcal{Y} \cup \bot$ and a corresponding distribution $\bP(Y^\bot|X=x)$ such that minimizing the misclassification loss on this new space will be equivalent to minimizing our  system loss $L_{0{-}1}$. The Bayes optimal classifier on  $\mathcal{Y}^\bot$ is clearly $h^\bot= \arg \max_{y^\bot\in\mathcal{Y}^\bot} \bP(\mathcal{Y}^\bot=y^\bot|X=x)$, and we need it to match the decision of the Bayes solution $h^B,r^B$ of $L_{0{-}1}$ \eqref{eq:01_reject_loss}:
\begin{equation}
    h^B, r^B = \arg\inf_{h,r} L_{0{-}1}(h,r)
\end{equation}
where the infimum is over all measurable functions.
Denote by $\eta_y(x)=\bP(Y=y|X=x)$, it is clear that for $x\in \mathcal{X}$ the best classifier is the same as the Bayes solution for standard classification since if we don't defer we have to do our best. Now we only reject the classifier if it's expected error is higher than the expected error of the expert which we formalize in the below proposition:
\begin{proposition}
The minimizers of the loss $L_{0{-}1}$ \eqref{eq:01_reject_loss} are defined point-wise for all $x\in \mathcal{X}$ as:
\begin{align}
    &h^B(x) = \arg \max_{y \in \mathcal{Y}}\eta_y(x)  \nonumber    \\  
    &r^B(x)= \bI_ {\max_{y \in \mathcal{Y}}\eta_y(x) \leq  \bP(Y = M|X=x) }  \label{propeq:bayes_01}
\end{align}
\end{proposition}
Proof of the above proposition can be found in Appendix \ref{apx:proofs} 
and  equation \eqref{propeq:bayes_01} give us sufficient conditions for consistency to check our proposed loss.   
Let $g_y: \mathcal{X} \to \mathbb{R}$ for $y \in \mathcal{Y}$ and define $h(x) = \arg \max_{y \in \mathcal{Y}}g_y$, similarly let  $g_\bot: \mathcal{X} \to \mathbb{R}$ and define $r(x)= \bI_{\max_{y \in \mathcal{Y}}g_y(x) \leq g_\bot }$  the proposed surrogate loss  for $L_{0{-}1}$ \eqref{eq:original_reject_loss} in the multiclass setting is then:
\begin{align} \label{eq:proposed_CE_loss}
& L_{CE}(h,r,x,y,m) = -  \log\left(\frac{\exp(g_{y}(x))}{\sum_{y' \in \mathcal{Y} \cup  \bot}\exp(g_{y'}(x))} \right) - \bI_{m = y} \log\left(\frac{\exp(g_{\bot}(x))}{\sum_{y' \in \mathcal{Y} \cup \bot}\exp(g_{y'}(x))} \right)
\end{align}
The proposed surrogate $L_{CE}$ is in fact consistent and upper bounds $L_{0{-}1}$ as the following theorem demonstrates.
\begin{theorem}
The loss $L_{CE}$ is convex in $\mathbf{g}$, upper bounds  $L_{0{-}1}$ and is consistent: $\inf_{h,r}\bE_{x,y,m}[L_{CE}(h,r,x,y,m)]$ is attained at $(h^*_{CE},r^*_{CE})$ such that $h^B(x)=h^*_{CE}(x)$ and $r^B(x)=r^*_{CE}(x)$ for all $x \in \mathcal{X}$.
\end{theorem}

\begin{sproof} Please refer to appendix \ref{apx:proofs} for the detailed proof. First the infimum over functions $h,r$ can be replaced by a point-wise infimum as:
\begin{flalign*}
&\inf_{h,r}\bE_{x,y,m}[L_{CE}(h,r,x,y,m)]
= \bE_{x}\inf_{h(x),r(x)}\bE_{y|x}\bE_{m|x,y}[L_{CE}(h(x),r(x),x,y,m)]&
\end{flalign*}
Now let us expand the inner expectation:
\begin{flalign}
\bE_{y|x}\bE_{m|x,y}[L_{SH}(h(x),r(x),x,y,m)]
&=- \sum_{y \in \mathcal{Y}} \eta_y(x) \log\left(\frac{\exp(g_{y}(x))}{\sum_{y' \in \mathcal{Y} \cup \bot}\exp(g_{y'}(x))} \right)   \label{the_body:y|x_cross_loss_withq}    \\&- \bP(Y= M |X=x)   \log\left(\frac{\exp(g_{\bot}(x))}{\sum_{y' \in \mathcal{Y} \cup \bot}\exp(g_{y'}(x))} \right) \nonumber
\end{flalign}

 For ease of notation denote the RHS of equation \eqref{the_body:y|x_cross_loss_withq} as $L_{CE}(g_1,\cdots,g_{|\mathcal{Y}|},g_\bot)$, note that it is a a convex function, hence we will take the partial derivatives with respect to each argument and set them to $0$.
For any $g_\bot$ and  $i \in \mathcal{Y}$ we have :
\begin{flalign}
&   \frac{\exp(g_{i}^*(x))}{\sum_{y' \in \tilde{\mathcal{Y}}}\exp(g_{y'}(x))} = \frac{\eta_i(x)}{1 +\bP(Y= M|X=x)}\label{the_body:optimal_g}
\end{flalign}
The optimal $h^*$ for any $g_\bot$ should satisfy equation \eqref{the_body:optimal_g} for every $i \in \mathcal{Y}$. 
Plugging $h^*$ and taking the derivative with respect to $g_\bot$ we get:

\begin{flalign}
&\nonumber \frac{\exp(g_{\bot}^*(x))}{\sum_{y' \in \mathcal{Y}}\exp(g_{y'}^*(x))} = \frac{\bP(Y= M|X=x)}{1+\bP(Y= M|X=x)}
\end{flalign}
since  exponential is an increasing function we get that the optimal $h^*$ and $r^*$ in fact agrees with the Bayes solution.
\end{sproof}
When the costs $c(1), \cdots, c(K+1)$ are in accordance with our  expert deferral setting the loss $\tilde{L}_{CE}$ reduces to $L_{CE}$. 
Now stepping back and looking more closely at our loss $L_{CE}$, we can see that the loss on examples where the expert makes a mistake becomes the cross entropy loss with the target. On the other hand, when the expert agrees with the target, the learner faces two opposing decisions whether to defer or predict the target. We can encourage or hinder the action of deferral by modifying the loss with an additional parameter $\alpha \in \mathbb{R}^+$ as $L_{CE}^\alpha(h,r,x,y,m)$:
\begin{align}
\nonumber  L_{CE}^\alpha(h,r,x,y,m)=& -( \alpha \cdot \bI_{m = y} + \bI_{m \neq y} )\log\left(\frac{\exp(g_{y}(x))}{\sum_{y' \in \mathcal{Y} \cup \nonumber \bot}\exp(g_{y'}(x))} \right) \\&- \bI_{m = y} \log\left(\frac{\exp(g_{\bot}(x))}{\sum_{y' \in \mathcal{Y} \cup \bot}\exp(g_{y'}(x))} \right) 
\end{align}
Note that $ L_{CE}^1 = L_{CE}$. The effect of $\alpha$ is to re-weight examples where the expert is correct to discourage the learner of fitting them and instead focus on examples where the expert makes a mistake. In practice, one would treat $\alpha$ as an additional hyperparameter to optimize for.

\section{Theoretical analysis}\label{sec:theory}
In this section we focus on the zero-one system loss function $L_{0{-}1}$ and try to understand previous proposed solutions in the literature in comparison with our method from a theoretical perspective.

\subsection{Failure of Confidence Scores Method}
Let us first remind ourselves of the Bayes solution for the system loss:
\begin{align*}
    &h^B(x) = \arg \max_{y \in \mathcal{Y}}\eta_y(x), \quad r^B(x)= \bI_ {\max_{y \in \mathcal{Y}}\eta_y(x) \leq  \bP(Y = M|X=x) } 
\end{align*}
  The form of the Bayes solution above suggests a very natural approach: 1)  learn a classifier minimizing the misclassification loss with the target and obtain confidence scores for predictions, 2) obtain confidence scores for expert agreement with the target, this can be done by learning a model where the target is whether the expert agrees with the task label and extracting confidence scores from this model \cite{raghu2019direct}, and finally 3) compare who between the classifier and the expert is more confident and accordingly defer. 
 We refer to this as the confidence score method (Confidence), this approach leads to a consistent estimator for both the rejector and classifier and was proposed by \cite{raghu2019algorithmic}.
 
 In fact this is the standard approach in rejection learning \cite{bartlett2008classification,ramaswamy2018consistent,ni2019possibility}, a host of different methods exist for estimating a classifier's confidence on new examples including trust scores \cite{jiang2018trust}, Monte-Carlo dropout for neural networks \cite{gal2016dropout} among many others. However, the key pitfall of this method in the expert deferral setup it that it does not allow $h$ to adapt to the expert's strengths and weaknesses.
When we restrict our search space to a limited class of functions $\mathcal{H}$ and $\mathcal{R}$ this approach can easily fail. 
We now give a  toy example where learning the classifier independently fails which  motivates the need to jointly learn both the classifier and rejector.
\begin{wrapfigure}{r}{0.50\textwidth}
\centering
\resizebox{0.50\textwidth}{!}{
  \includegraphics[clip,scale=0.7,trim={0.0cm 9.3cm 21.0cm 0.0cm}]{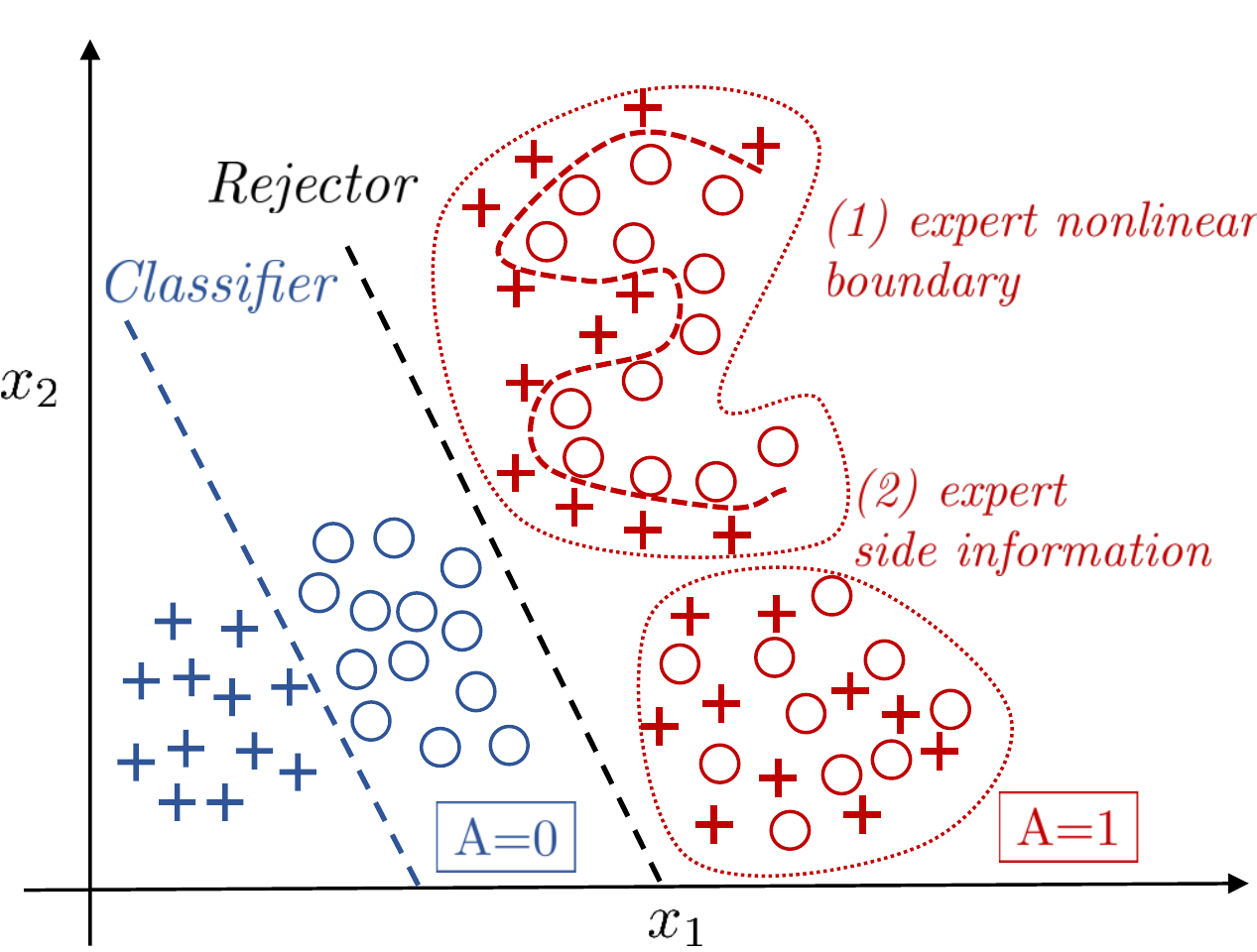}
  }
    \caption{Setting of two groups, red and blue, the task is binary classification with labels $\{o,+\}$, the expert fits the red majority group, hence the classifier should attempt to fit the blue group with the rejector (black line) separating the groups.}
    \label{fig:expert_gauss}
\end{wrapfigure}

 Assume that there exists two  sub-populations in the data denoted $A=1$ and $A=0$ where $\bP(A=1)\geq \bP(A=0)$ from which $X \in \mathbb{R}^d$ is generated from and conditional on the target and population, $X|(Y=y,A=0)$ is normally distributed according to $ \mathcal{N}(\mu_{y,0}, \Sigma)$ and $X|(Y=y,A=1)$ consists of two clusters: cluster (1) is normally distributed but the means are not well separated and cluster (2) is only separable by a complex non-linear boundary; the data is illustrated in Figure \ref{fig:expert_gauss}. Finally we assume the expert to be able to perfectly classify group $A=1$, on cluster (1) the expert is able to compute the complex nonlinear boundary and on cluster (2) the expert has side-information $Z$ that allows him to separate the classes which is not possible from only $X$.  Our hypothesis spaces $\mathcal{H}$ and $\mathcal{G}$ will be the set of all $d-$dimensional hyperplanes.
If we start by learning $h$, then the resulting hyperplane will try to minimize the average error across both groups, this will likely result into a hyperplane that separates neither group as the data is not linearly separable, especially on group $A=1$. 
If we assume that the boundary between the groups is linear as shown, then we can achieve the error of the Bayes solution within our hypothesis space: the optimal behavior in this setting is clearly to have $h$ fit group $A=0$, note here the Bayes solution corresponds to a hyperplane via linear discriminant analysis for $2$ classes on $A=0$, and the rejector $r$ separating the groups as illustrated in Figure \ref{fig:expert_gauss}. This example illustrates the complexities of this setting, due to model capacity there are significant gains to be achieved from adapting to the expert by focusing only group $A=0$. Setting aside model capacity, the nonlinear boundary of cluster (1) is sample intensive to learn as we only have access to finite data. Finally, cluster (2) cannot be separated even with infinite data, the side information of the expert is needed, and so the hard task is to identify the region of cluster (2).
This serves to illustrates the complexities of the setup and the importance of learning the classifier and rejector jointly.
\subsection{Inconsistency of mixtures of experts loss  and Realizable-consistency}

So far we have focused on classification consistency to verify the soundness of  proposed approaches, however, we usually have specific hypothesis classes $\cH, \cR$ in mind, and if the Bayes predictor is not in our class then consistency might not guarantee much \cite{ben2012minimizing}. For example, for binary classification with half-spaces, any predictor learned with a convex surrogate loss can have arbitrarily high error if the best half-space has non-zero error \cite{ben2012minimizing}. The previous example illustrated in Figure \ref{fig:expert_gauss} shows an the mode of failure that exists in the expert deferral setup even in the realizable setting. 
Therefore, a more relevant requirement to our example is that the minimizers of a proposed surrogate loss and the original loss agree for given hypothesis classes in the \emph{realizable} setting; this is formally defined with the below notion. 

\begin{definition}[realizable $(\mathcal{H},\mathcal{R})$-consistency]
A surrogate loss $L_{surr}$ is realizable $(\mathcal{H},\mathcal{R})$-consistent if for all distributions $\mathbf{P}$ and experts $M$ for which there exists $h^*,r^* \in \mathcal{H} \times \mathcal{R}$ that have zero error $L(h^*,r^*)=0$, we have $\forall \epsilon >0$, $\exists \delta >0$ such that if $(\hat{h},\hat{r})$ satisfies
\begin{center}

$\left| L_{surr}(\hat{h},\hat{r}) - \inf_{h \in \cH, r \in \cR}  L_{surr}(h,r) \right| \leq \delta$, then: 
$ L(\hat{h},\hat{r})  \leq \epsilon$
\end{center}

\end{definition}
A similar notion was introduced for classification by \cite{long2013consistency} and by \cite{cortes2016learning} for rejection learning, however here we have the the added dimension of the  expert. 
 
Note that the expert deferral setting considered here can be thought of as a hard mixture of two experts problem where one of the experts is fixed \cite{jordan1994hierarchical,shazeer2017outrageously,madras2018predict}. This observation motivates a natural mixture of experts type loss, let $g_y: \mathcal{X} \to \mathbb{R}$ for $y \in \mathcal{Y}$, $h(x) = \arg \max_{y \in \mathcal{Y}}g_y$,  $r_i: \mathcal{X} \to \mathbb{R}$ for $i \in \{0,1\}$ and  $r(x)= \arg\max_{i \in \{0,1\}} r_i(x)$, the mixture of experts loss is defined as:

\begin{align}
& L_{mix}(\mathbf{g},\mathbf{r},x,y,m) = -\log\left(\frac{\exp(g_{y}(x))}{\sum_{y' \in \mathcal{Y}}\exp(g_{y'}(x))} \right) \frac{\exp(r_{0}(x))}{\sum_{i \in \{0,1\}}\exp(r_{i}(x))}  + \bI_{m \neq y} \frac{\exp(r_{1}(x))}{\sum_{i \in \{0,1\}}\exp(r_{i}(x))}   \label{eq:mix_of_exp_loss}
\end{align}
The above loss extends \cite{madras2018predict} approach to the multiclass setting. As the next proposition demonstrates, $L_{mix}$ is in general \emph{not} classification consistent, however, it is realizable $(\mathcal{H},\mathcal{R})$-consistent for classes closed under scaling which include linear models and neural networks.

\begin{proposition} 
$L_{mix}$ is realizable $(\mathcal{H},\mathcal{R})$-consistent for classes closed under scaling but is not  classification consistent.
\label{prop:realizable_madras}
\end{proposition}

Proof of proposition \ref{prop:realizable_madras} can be found in Appendix \ref{apx:proofs}.
Note that integrating more information about $M$ in $L_{mix}$ would not make the loss consistent, the inconsistency arises from the parameterization in $\bm{g}$, setting the classifier loss to simply be $\bI_{h(x)\neq y}$ would make $L_{mix}$ consistent at the cost of losing the convexity and differentiability in $\bm{g}$.
While $L_{mix}$ is indeed realizable consistent however it is not convex in both $\bm{g}$ and $\bm{r}$, hence it is not clear how to efficiently optimize it. 
Setting aside computational feasibilities, it is also not immediately clear which between consistency and realizable $(\cH,\cR)$-consistency will be more practically relevant. In our experimental section we show how the mismatch between the model and expert loss and their actual errors causes this method to learn the incorrect behavior which hints that classification consistency is crucial.



\subsection{Generalization Bound For Joint Learning}
In this subsection we analyze the sample complexity to jointly learn a rejector and classifier.  
The goal is to find the minimizer of the empirical version of our system loss when our hypothesis space for $h$ and $r$ are $\mathcal{H},\mathcal{R}$ respectively:
\begin{equation}
    \hat{h}^*,\hat{r}^* = \arg\min_{h \in \mathcal{H}, r \in \mathcal{G}} L^S_{0{-}1}(h,r):=
\frac{1}{n} \sum_{i=1}^n \ \bI_{h(x_i) \neq y_i} \bI_{r(x_i) = 0} +  \bI_{m_i \neq y_i}\bI_{r(x_i)=1}
\end{equation}
By going after the system loss directly, we can approximate the population minimizers $h^*,r^*$ over $\mathcal{H} \times \mathcal{R}$ of $L_{0{-}1}$ \eqref{eq:01_reject_loss}. The optimum $h^*$ may not necessarily coincide with the optimal minimizer of the misclassification loss with the target which is why learning jointly is critical. We now give a generalization bound for our empirical minimization procedure for a binary target.

\begin{theorem}
 For any expert $M$ and data distribution $\mathbf{P}$ over $\mathcal{X} \times \mathcal{Y}$, let $0<\delta<\frac{1}{2}$, then  with probability at least $1-\delta$, the following holds for the empirical minimizers $(\hat{h}^*,\hat{r}^*)$:
\begin{align}
    L_{0{-}1}(\hat{h}^*,\hat{r}^*) &\leq  L_{0{-}1}(h^*,r^*) + \mathfrak{R}_n(\mathcal{H}) +  \mathfrak{R}_{n}(\mathcal{R})  + \mathfrak{R}_{n \bP(M \neq Y)/2}(\mathcal{R})  \nonumber \\
    & + 2\sqrt{\frac{\log{\frac{2}{\delta}}}{2n}} +\frac{\bP(M\neq Y)}{2}  \exp\left(- \frac{n \bP(M \neq Y)}{8} \right)   \label{eq:the_generalization_bound}
\end{align}
\end{theorem}
Proof of the above theorem can be found in Appendix \ref{apx:proofs}. We can see that the performance of our empirical minimizer is controlled by the Rademacher complexity $\mathfrak{R}_{n}(\mathcal{R})$ and $\mathfrak{R}_{n}(\mathcal{H})$ of both the classifier and rejector model classes and the error of the expert.
Note that when $\bP(M \neq Y)=0$ we recover the bound proved in Theorem 1 \cite{cortes2016learning} for rejection learning when $c=0$; this gives evidence that deferring to an expert is a more sample intensive problem then rejection learning. Both our loss $L_{CE}$ and the confidence scores approach lead to consistent estimators, however, as we will later show in our experiments, one differentiating factor will be that of sample complexity. We can already see in the bound \eqref{eq:the_generalization_bound}, that we pay  the complexity of the rejector and classifier model classes, however, our approach combines the rejector and classifier in one model to avoid these added costs.

\section{Experiments}\label{sec:experiments}
We provide code to reproduce our experiments \footnote{\url{https://github.com/clinicalml/learn-to-defer}}. In Appendix \ref{apx:guide} we give a detailed guide on implementing our method. Additional experimental details and results are left to Appendix \ref{apx:experiments}.



\subsection{Synthetic Data}

As a first toy example to showcase that our proposed loss $L_{CE}^\alpha$ is able to adapt to the underlying expert behavior, we perform experiments in a Gaussian mixture setup akin to the example in section \ref{sec:theory}. 
The covariate space is $\mathcal{X}= \mathbb{R}^d$ and target $\mathcal{Y} = \{0,1\}$, we assume that there exists two  sub-populations in the data denoted $A=1$ and $A=0$.
Furthermore, $X|(Y=y,A=a)$ is normally distributed according to $ \mathcal{N}(\mu_{y,a}, \Sigma_{y,a})$.
The expert follows the Bayes solution for group $A=1$ which here corresponds to a hyperplane. Our hypothesis spaces $\mathcal{H}$ and $\mathcal{R}$ will be the set of all $d-$dimensional hyperplanes.

\textbf{Setup:} We perform 200 trials where on each trial we generate: random group proportions $\bP(A=1) \sim U(0,1)$ fixing $\bP(Y=1|A=a)=0.5$, random means and variances for each Gaussian component $X|Y=y,A=a \sim \mathcal{N}(\mu_{y,a}, \Sigma_{y,a})$ where $\mu_{y,a} \sim U(0,10)^d$ and similarly for the diagonal components of $\Sigma_{y,a}(i,i) \sim U(0,10)$ keeping non-diagonal components $0$ with dimension $d=10$; we generate in total $1000$ samples each for training and testing. We compare against oracle behavior and two baselines: 1) An oracle baseline (Oracle) that trains only on $A=0$ data and trains the rejector to separate the groups with knowledge of group labels and 2) the confidence score baseline (Confidence) that trains a linear model on all the data and then trains a different linear model on all the data where labels are the expert's agreement with the target and finally compares which of the two is more confident according to the probabilities assigned by the corresponding models and 3) our implementation of the approach in \cite{madras2018predict} (MixOfExp). 

\textbf{Results:} We train a multiclass logistic regression model with our loss $L_{CE}^\alpha$ with $\alpha \in \{0,0.5,1\}$ and record in table \ref{table:guassian} the difference in accuracy between our method and baselines for the best performing $\alpha$. We can see that our method with $\alpha=0$ outperforms the confidence baseline by $6.39$ on average in classification accuracy and matches the oracle method  with $0.22$ positive difference which shows the success of our method. When trained with loss $L_{CE}^1$ or $L_{CE}^{.5}$ the model matches the confidence baseline, the reason being is that with $\alpha \neq 0$ the model will still try to fit the target $Y$ but the model class here is not rich enough to allow the model to reasonably fit the target and adapt to the expert.

\begin{table}[H]
\caption{Comparison of our methods with the confidence score baseline, oracle baseline and our implementation of \cite{madras2018predict} method. We compute a 95\% confidence interval for the average difference between the baselines and our method.}
\label{table:guassian}
\vskip 0.15in
\begin{center}
\begin{small}
\begin{sc}
\begin{tabular}{lcr}
\toprule
Difference in system accuracy & Average & 95\%  interval  \\
\midrule
$L_{CE}^0$-Confidence \cite{raghu2019algorithmic}    & 6.39 & [3.71,9.06] \\
$L_{CE}^0$-Oracle    & 0.22 & [-1.71,2.15] \\
$L_{CE}^0$- MixOfExp \cite{madras2018predict}    & 2.01 & [0.14,4.06] \\
\bottomrule
\end{tabular}
\end{sc}
\end{small}
\end{center}
\vskip -0.1in
\end{table}

\subsection{CIFAR-10}
As our first real data experimental evaluation we conduct experiments on the celebrated CIFAR-10 image
classification dataset \cite{krizhevsky2009learning} consisting of $32 \times32$ color
images drawn from 10  classes split into 50,000 train and 10,000 test images.

\textbf{Synthetic Expert.} We simulate multiple synthetic experts of varying competence in the following way: let $k \in [10]$, then if the image belongs to the first $k$ classes the expert predicts perfectly, otherwise the expert predicts uniformly over all classes. The classifier and expert costs are assumed to be the misclassification costs.

\textbf{Base Network.} Our base network for classification will be the Wide Residual Networks (WideResNets) \cite{zagoruyko2016wide} which with data augmentation and hyperparameter tuning can achieve a $96.2\%$ test accuracy. Since our goal is not to achieve better accuracies but to show the merit of our approach for a given fixed model, we disadvantage the model by not using data augmentation and a smaller network size. The WideResNet with 28 layers minimizing the cross-entropy loss achieves $90.47\%$ test accuracy with training until fitting the data in 200 epochs; this will be our benchmark model. We use SGD with momentum and a cosine annealing learning rate schedule.

\textbf{Proposed Approach:} Following section 4, we parameterize $h$ and $r$ (specifically $g_\bot$) by a WideResNet with $11$ output units where the first $10$ units represent $h$ and the $11'th$ unit is $g_\bot$ and minimize the proposed surrogate $L_{CE}^\alpha$ \eqref{eq:proposed_CE_loss}. We also experimented with having $h$ be a WideResNet with $10$ output units and $g_\bot$  a WideResNet with a single output unit and observed identical results. 
We show results for $\alpha \in \{0.5,1\}$.

\textbf{Baselines:} We compare against three baselines. The first baseline trains the rejector to recognize if the image is in the first $k$ classes and accordingly defers, we call this baseline "LearnedOracle"; this rejector is a learned implementation of what the optimal rejector should do. The second baseline is the confidence score method \cite{raghu2019algorithmic} and the third is the mixture-of-experts loss of \cite{madras2018predict}, details of the implementation of this final baseline are left to Appendix \ref{apx:madras}.

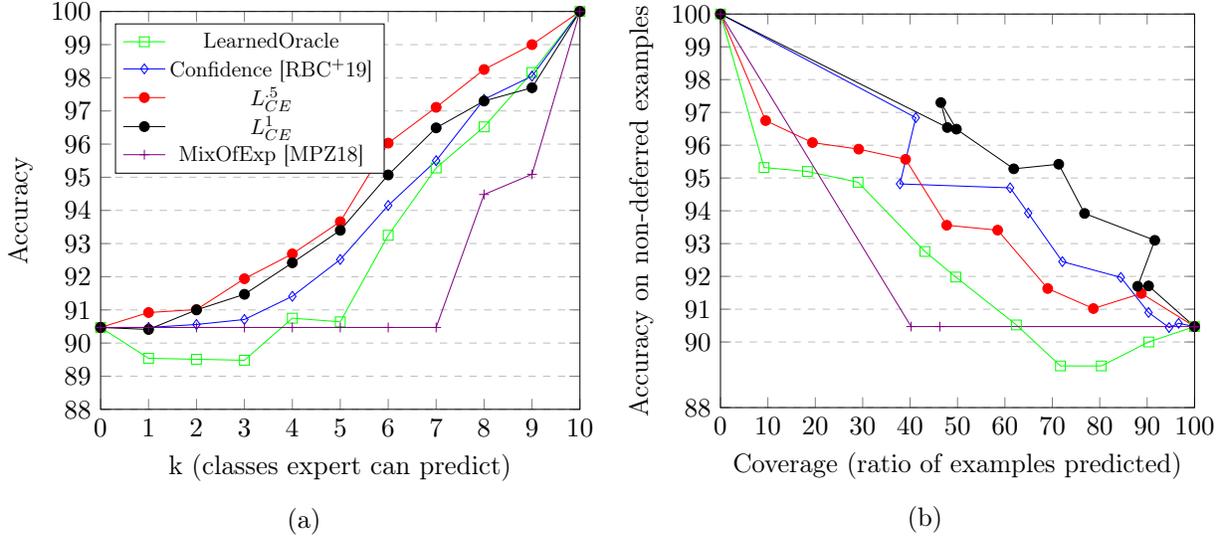
\begin{figure}
\centering
\begin{subfigure}{.5\textwidth}
\centering
\resizebox{3.2in}{!}{%
\begin{tikzpicture}
\begin{axis}[
    title={},
    xlabel={k (classes expert can predict)},
    ylabel={Accuracy},
    xmin=0, xmax=10,
    ymin=88, ymax=100,
    xtick={0,1,2,3,4,5,6,7,8,9,10},
  ytick={88,89,90,91,92,93,94,95,96,97,98,99,100},
    legend pos=north west,
legend style={nodes={scale=0.8, transform shape}},
    ymajorgrids=true,
    grid style=dashed,
]

\addplot[
    color=green,
    mark=square,
    ]
    coordinates {
    (0,90.47)(1,89.54)(2,89.51)(3,89.48)(4,90.75)(5,90.64)(6,93.25)(7,95.28)(8,96.52)(9,98.16)(10,100)
    };
    \addlegendentry{LearnedOracle}

\addplot[
    color=blue,
    mark=diamond,
    ]
    coordinates {
    (0,90.47)(1,90.47)(2,90.56)(3,90.71)(4,91.41)(5,92.52)(6,94.15)(7,95.5)(8,97.35)(9,98.05)(10,100)
    };
    \addlegendentry{Confidence \cite{raghu2019algorithmic}}
    
    \addplot[
    color=red,
    mark=*,
    ]
    coordinates {
    (0,90.47)(1,90.92)(2,91.01)(3,91.94)(4,92.69)(5,93.66)(6,96.03)(7,97.11)(8,98.25)(9,99)(10,100)
    };
    \addlegendentry{$L_{CE}^{.5}$}
    
        \addplot[
    color=black,
    mark=*,
    ]
    coordinates {
    (0,90.47)(1,90.41)(2,91)(3,91.47)(4,92.42)(5,93.4)(6,95.07)(7,96.49)(8,97.3)(9,97.7)(10,100)
    };
    \addlegendentry{$L_{CE}^{1}$}
        \addplot[
    color=violet,
    mark=+,
    ]
    coordinates {
    (0,90.47)(1,90.47)(2,90.47)(3,90.47)(4,90.47)(5,90.47)(6,90.47)(7,90.47)(8,94.48)(9,95.09)(10,100)
    };
    \addlegendentry{MixOfExp \cite{madras2018predict}}

\end{axis}
\end{tikzpicture}
}%
\caption{
}
\label{fig:kvssystem2}
\end{subfigure}%
\begin{subfigure}{.5\textwidth}
\centering
\resizebox{3.2in}{!}{%
\begin{tikzpicture}
\begin{axis}[
    title={},
    xlabel={Coverage (ratio of examples predicted)},
    ylabel={Accuracy on non-deferred examples},
    xmin=0, xmax=100,
    ymin=88, ymax=100,
    xtick={0,10,20,30,40,50,60,70,80,90,100},
    ytick={88,90,91,92,93,94,95,96,97,98,99,100},
    ymajorgrids=true,
    grid style=dashed,
]
\addplot[
    color=green,
    mark=square,
    ]
    coordinates {
    (100,90.47)(90.34,90)(80.33,89.27)(71.71,89.27)(62.43,90.52)(49.65,91.98)(43.12,92.76)(29.03,94.87)(18.34,95.2)(9.2,95.32)(0,100)
    };
    \addlegendentry{LearnedOracle}

\addplot[
    color=blue,
    mark=diamond,
    ]
    coordinates {
    (100,90.47)(96.64,90.57)(94.63,90.44)(90.29,90.9)(84.45,91.97)(72.12,92.45)(64.91,93.93)(61.1,94.7)(37.86,94.82)(41.2,96.84)(0,100)
    };
    \addlegendentry{Confidence \cite{raghu2019algorithmic}}
    
    \addplot[
    color=red,
    mark=*,
    ]
    coordinates {
    (100,90.47)(88.78,91.48)(78.66,91.02)(69.02,91.63)(58.47,93.41)(47.72,93.56)(39.05,95.57)(29.15,95.88)(19.41,96.08)(9.53,96.75)(0,100)
    };
    \addlegendentry{$L_{CE}^{.5}$}
    
        \addplot[
    color=black,
    mark=*,
    ]
    coordinates {
    (100,90.47)(90.3,91.71)(88,91.7)(91.61,93.1)(76.79,93.92)(71.35,95.42)(61.86,95.28)(49.77,96.49)(46.48,97.3)(47.85,96.54)(0,100)
    };
    \addlegendentry{$L_{CE}^{1}$}
         \addplot[
    color=violet,
    mark=+,
    ]
    coordinates {
    (100,90.47)(100,90.47)(100,90.47)(100,90.47)(100,90.47)(100,90.47)(100,90.47)(100,90.47)(46.27,90.47)(40.22,90.47)(0,100)
    };
    \addlegendentry{MixOfExp \cite{madras2018predict}}   
   \legend{};
\end{axis}
\end{tikzpicture}
}%
\caption{
}
\label{fig:covvsacc}
\end{subfigure}
\caption{Left figure shows overall system accuracy of our method and baselines (k is the number of classes the expert can predict) and right figure compares the accuracy on the non-deferred examples versus the coverage for every $k$ }
\label{fig:cifar10_2figs}
\end{figure}

\textbf{Results.} In figure \ref{fig:kvssystem} we plot the accuracy of the combined algorithm and expert system versus $k$, the number of classes the expert can predict perfectly. We can see that the model trained with $L_{CE}^{0.5}$ and $L_{CE}^1$  outperforms the  baselines by $1.01$\% on average for the confidence score baseline and by $1.94$ on average for LearnedOracle. To look more closely at the behavior of our method, we plot in figure \ref{fig:covvsacc} the accuracy on the non-deferred examples versus the coverage, the fraction of the examples non-deferred, for each $k$. We can see that that the model trained with $L_{CE}^1$ dominates all other baselines giving better coverage and accuracy for the classifier's predictions. This gives evidence that our loss allows the model to only predict when it is highly confident.

\textbf{Why do we outperform the baselines?}

1) \emph{Sample complexity}: The Confidence baseline \cite{raghu2019algorithmic} requires training two networks while ours only requires one, when data is limited our approach gives significant improvements in comparison. We experiment with increasing training set sizes while keeping the test set fixed and training our model with  $L_{CE}^1$ and the Confidence baseline. Figure \ref{fig:dataregimes} plots system accuracy versus training set size when training with expert $k=5$. We can see when data is limited our approach massively improves on the baseline, for example with $2000$ training points, Confidence achieves $62.33$\% accuracy while our method achieves $70.12$\%, a $7.89$ point increase.

2) \emph{Taking into consideration both expert and model confidence}: the LearnedOracle baseline ignores model confidence entirely and only focuses on the region where the expert is correct. While this is the behavior of the Bayes classifier in this setup, when dealing with a limited model class and limited data, this no longer is the correct behavior. For this reason, our model outperforms the LearnedOracle baseline.

3) \emph{Consistency}: the mixtures of experts loss of \cite{madras2018predict} fails in this setup and learns never to defer. The reason is that when training, the loss of the classifier will converge to zero and validation classifier accuracy will still improve in the mean-time, however the loss of the expert remains constant, thus we never defer.

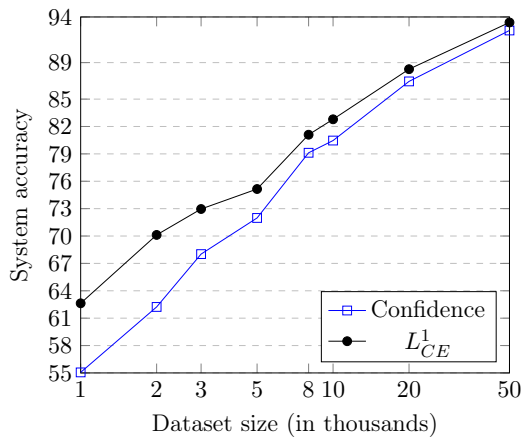
\begin{figure}[h]
\centering
\resizebox{2.8in}{!}{%
\begin{tikzpicture}
\begin{semilogxaxis}[
    title={},
    xlabel={Dataset size (in thousands)},
    ylabel={System accuracy},
    xmin=1, xmax=50,
    ymin=55, ymax=94,
    xtick={1,2,3,5,8,10,20,50},
    ytick={55,58,61,64,67,70,73,76,79,82,85,89,94},
    legend pos=south east,
    log ticks with fixed point,
    ymajorgrids=true,
    yminorticks=false,
    grid style=dashed,
]
 
\addplot[
    color=blue,
    mark=square,
    ]
    coordinates {
    (1,55.06)(2,62.23)(3,68.02)(5,71.98)(8,79.12)(10,80.46)(20,86.94)(50,92.52)
    };
    \addlegendentry{Confidence}

        \addplot[
    color=black,
    mark=*,
    ]
    coordinates {
    (1,62.63)(2,70.12)(3,72.96)(5,75.15)(8,81.1)(10,82.79)(20,88.28)(50,93.4)
    };
    \addlegendentry{$L_{CE}^{1}$}
\end{semilogxaxis}
\end{tikzpicture}
}%
\caption{Varying training set size when training with expert $k=5$ for Confidence baseline and our method $L_{CE}^1$.}
\label{fig:dataregimes}
\end{figure}

\subsection{CIFAR-100}
We repeat the experiments described above on the CIFAR-100 dataset \cite{krizhevsky2009learning}. A 28 layer WideResNet achieves a 79.28 \% test accuracy when training with data augmentation (random crops and flips). The simulated experts also operate in a similar fashion, for $k \in \{10,20,\cdots,100\}$, if the image is in the first $k$ classes, the expert predicts the correct label with probability $0.94$ to simulate SOTA performance on CIFAR-100 with 93.8\% test accuracy \cite{kolesnikov2019large}, otherwise the expert predicts uniformly at random. 

Compared against the confidence score baseline, the model trained with $L_{CE}^1$ outperforms it by a  1.60 difference in test accuracy for $30 \leq k\leq 90 $ on average and otherwise performs on par. This gives again gives evidence for the efficacy of our method, full experimental results are available in appendix \ref{apx:expcifar100}. 


\subsection{CIFAR10H and limited expert data}
Obtaining expert labels for entire datasets may in fact be prohibitively expensive as standard dataset sizes have grown into million of points \cite{deng2009imagenet}. 
Therefore it is more realistic to expect that the expert has labeled only a fraction of the data. In the following experiments we assume access to fully labeled data $S_l = \{(x_i,y_i,m_i)\}_{i=1}^m$ and data without expert labels $S_u=\{ (x_i,y_i)\}_{i=m+1}^n$.  The goal again is to learn a classifier $h$ and rejector $r$ from the two datasets $S_l$ and $S_u$.

\textbf{Data:} To experiment in settings where we have limited expert data, we use the dataset \texttt{CIFAR10H} \cite{peterson2019human} initially developed to improve model robustness. \texttt{CIFAR10H} contains for each data point in the CIFAR-10 test set fifty crowdworker annotations recorded as counts for each of the 10 classes. The training set of CIFAR-10 will constitute $S_u$, and we randomly split the test set in half where one half constitutes $S_l$ and the other is for testing; we randomize the splitting over 10 trials.

\textbf{Expert:} We simulate the behavior of an average human annotator by sampling  from the class counts for each data point. The performance of our simulated expert has an average classification accuracy of 95.22 with a standard deviation of 0.18 over 100 runs.  The performance of the expert is non uniform over the classes, for example on the class \textit{cat} the expert has 91.0\% accuracy while on \textit{horse} a 97.8\% accuracy. 

\textbf{Proposed Approach:} 
Our method will be to learn $f_m: \mathcal{X} \to \{0,1 \}$ to predict whether the expert errs from data $\tilde{S}_l = \{(x_i,\bI_{y_i \neq m_i})\}_{i=1}^m$, using $f_m$ we label $S_u$ with the expert disagreement labels to use in our loss function  an obtain $\hat{S}_u$. Note since our loss function does not care which label the expert predicts but whether he errs or not, our task simplifies to binary classification instead of classification over the target $\mathcal{Y}$. Finally we train using our loss $L_{CE}$ on $\hat{S}_u \cup S_l$; we refer to our method as "$L_{CE}$ impute"


\begin{table}[h]
\caption{Comparing our proposed methods on \texttt{CIFAR10H} and a baseline based on confidence scores recording system accuracy, coverage and classifier accuracy on non-deferred examples.}
\label{table:cifar10h}
\vskip 0.15in
\begin{center}
\begin{small}
\begin{sc}

\begin{tabular}{lccr}
\toprule
Method & System  & Coverage & Classifier   \\
\midrule
$L_{CE}$ impute   & \textbf{96.29}$\pm$0.25 & 51.67$\pm$1.46 &\textbf{ 99.2} $\pm$ 0.08 \\
$L_{CE}$ 2-step   & 96.03$\pm$0.21 & 60.81$\pm$0.87 & 98.11 $\pm$ 0.22 \\
Confidence \cite{raghu2019algorithmic}   & 95.09$\pm$0.40 & \textbf{79.48}$\pm$5.93 & 96.09 $\pm$ 0.42 \\
\bottomrule
\end{tabular}

\end{sc}
\end{small}
\end{center}
\vskip -0.1in
\end{table}

\textbf{Results.} We compare against a confidence score baseline where we train a classifier on $S_u$ and then model the expert on $S_l$. Results are shown in table \ref{table:cifar10h} and we can see that our method outperforms the confidence method by $1.2$ points on system accuracy and an impressive $3.1$ on data points where the classifier has to predict. To show the effect of imputing expert labels on $S_u$, we train first our model using $L_{CE}$ on $S_u$ and then fine tune to learn deferral on $S_l$, we refer to this as "$L_{CE}$ 2-step". It is possible that further approaches inspired by SOTA methods in semi supervised learning methods give further improvements
\cite{oliver2018realistic,berthelot2019mixmatch}.

\subsection{Hate Speech and Offensive Language Detection}
We conduct experiments on the dataset created by
\cite{davidson2017automated}  consisting of 24,783 tweets annotated as hate
speech, offensive language or neither. 
We create a synthetic expert that has differing error rates according to the demographic of the tweet's author as described in what follows.

\textbf{Expert.} 
\cite{blodgett-etal-2016-demographic} developed a probabilistic language model  that can identify if a tweet is in African-American English (AAE), this model was used by \cite{davidson2019racial} to audit for racial bias in classifiers. We use the same model and predict that a tweet is in AAE if the probability predicted is higher than $0.5$. 
Our expert model is as follows: if the tweet is in AAE then with probability $p$ we predict the correct label and otherwise predict uniformly at random. On the other hand if the tweet is not in AAE, we predict with probability $q$ the correct label. We experiment with 3 different expert probabilities for $p$ and $q$: 1) a fair expert with $\{p=0.9,q=0.9\}$, 2) a biased expert towards AAE tweets $\{p=0.75,q=0.9\}$ and 3) a biased expert towards non AAE tweets $\{p=0.9,q=0.75\}$.

\textbf{Our Approach.} For our model we use the CNN developed in \cite{kim2014convolutional} for text classification with 100 dimensional Glove embeddings \cite{pennington2014glove} and $300$ filters of sizes $\{3,4,5\}$ using dropout. This CNN achieves a 89.5\% average accuracy on the classification task, comparable to the 91\% achieved by \cite{davidson2017automated} with a feature heavy linear model.
 We randomly split the dataset with a $60,10,30$\% split into a training, validation and test set respectively; we repeat the experiments for 5 random splits. We used a grid search over the validation set to find $\alpha$. 

\textbf{Results.} We compare against two baselines: the first is Confidence, the second is an oracle baseline that trains first a model on the classification task and then implements the Bayes rejector $r^B(x)$ equipped with the knowledge of $p,q$ and the tweet's demographic group.
Both our model trained with $L_{CE}^1$ and the confidence score baseline achieve similar accuracy and coverage with the oracle baseline performing only slightly better across the three experts. For the AAE biased expert, our model trained with $L_{CE}^1$ achieves 92.91$\pm$0.17 system accuracy, Confidence 92.42$\pm$0.40 and Oracle 93.22$\pm$0.11. This suggests that both approaches are performing optimally in this setting.

\textbf{Racial Bias.} A major concern in  this setting is whether the end to end system consisting of the classifier and expert is discriminatory. We define the discrimination of a predictor as the difference in the false positive rates of AAE tweets versus non AAE tweets where false positives indicate tweets that were flagged as hate speech or offensive when they were not. Surprisingly, the confidence score baseline with the fair expert doubles the discrimination of the overall system compared to the classifier acting on it's own: the classifier has a discrimination of $0.226$ on all the test data, the fair expert a discrimination of $0.03$ while the confidence score baseline has a discrimination of $0.449$. This again reiterates the established fact that fairness does not compose \cite{dwork2018fairness}. In fact, the end-to-end system can be less discriminatory even if the individual components are more discriminatory, for the second expert that has higher error rates on non AAE tweets with discrimination of $0.084$, the discrimination of the confidence score method reduces to $0.151$. While our method does not achieve significantly lower discrimination than the baseline, however integrating fairness constraints for the end-to-end system becomes easier as we can adapt the classifier. Complete experimental results can be found in Appendix \ref{apx:exphate}.


\subsection{Synthetic Experts on CheXpert}\label{subsec:chexpert}
\subsubsection{Setup}
\textbf{Task.} CheXpert is a large chest radiograph dataset that contains over 224 thousand images of 65,240 patients automatically labeled for the presence of 14 observations using radiology reports \cite{irvin2019chexpert}. In addition to the automatically labeled training set, \cite{irvin2019chexpert} make publicly accessible a validation set of 200 patients labeled by a consensus of 3 radiologists and hide a further testing set of 500 patients labeled by 8 radiologists. We focus here on the detection of only the 5 observations that make up the "competition tasks" \cite{irvin2019chexpert}: Atelectasis, Cardiomegaly,  Consolidation, Edema, and Pleural Effusion. This is a multi-task problem, we have 5 separate binary tasks, we will learn to defer on an individual task basis.

\textbf{Expert.} We create a simulated expert as follows: if the chest X-ray contains support devices (the presence of support devices is part of the label) then the expert is correct with probability $p$ on all tasks independently and if the X-ray does not contain support devices, then the expert is correct with probability $q$. We vary $q \in \{0.5,0.7\}$ and  $p \in \{0.7,0.8,0.9,1\}$ to obtain different experts, we let $p\geq q$ as one can think that a patient that has support devices might have a previous medical history that the expert is aware of and can use as side-information.

\textbf{Data.} We use the downsampled resolution version of CheXpert \cite{irvin2019chexpert} and split the training data set with an 80-10-10 split on a patient basis for training, validation and testing respectively, no patients are shared among the splits. Images are normalized and resized to be compatible with pre-trained ImageNet models, we use data augmentation in the form of random resized crops, horizontal flips and random rotations of up to $15\degree$  while training. Note that
a small subset of the training data has an uncertainty label "U" instead of a binary label that implies that the automatic annotator is uncertain, we ignore these points on a task basis while training and testing.

\textbf{Baselines.} We implement two baselines: a threshold confidence baseline that learns a threshold to maximize system AU-ROC on just the confidence of the classifier model to defer (ModelConfidence), this is the post-hoc thresholding method in \cite{madras2018predict}, and the Confidence baseline \cite{raghu2019algorithmic}. We use temperature scaling \cite{guo2017calibration} to ensure calibration of all baselines on the validation set.

\textbf{Model.} Following \cite{irvin2019chexpert}, we use the DenseNet121 architecture for our model with pre-trained weights on ImageNet, the loss for the baseline models is the average of the binary cross entropy for each of the tasks. We train the baseline models using Adam for 4 epochs. For our approach we train for 3 epochs using the cross entropy loss  and then train for one epoch using $L_{CE}^\alpha$ with $\alpha$ chosen to maximize the area under the receiver operating characteristic curve (AU-ROC) of the combined system on the validation set for each of the 5 tasks (each task is treated separately). We also observe similar results if we train for the first three epochs with $L_{CE}^1$ and then train for one epoch with a validated choice of $\alpha$.

\textbf{Experimental setup.} In a clinical setting there might be a cost associated to querying a radiologist, this then imposes a constraint on how often we can query the radiologist i.e. our model's coverage (fraction of examples where algorithm predicts). We constrain our method and the baselines to achieve $c\%$ coverage for $c \in [100]$ to simulate the spectrum between complete automation and none.\\
We achieve this for our method by first sorting the test set based on  $g_{\bot}(x) - \max(g_0(x),g_1(x)):=q(x)$ across all patients $x$ in the test set, then to achieve coverage $c$, we define $\tau = q(x_c)$ where $q(x_c)$ is the $c$'th percentile of the outputs $q(x)$, then we let $r(x) =1 \iff q(x) \geq \tau$. The definition of $\tau$ ensures that we obtain exactly $c\%$ coverage.

For ModelConfidence we achieve this by letting $q(x) = 1- \max(g_0(x),g_1(x))$ ($g$ is the result of a separate trained model than the one for our method), this is the natural classifier's probability of error from the softmax output, and for the Confidence we let $q(x)$ be the difference between the radiologists confidence and  the classifier's confidence.
\subsubsection{Results}
\textbf{Results.} In Figure \ref{fig:auc_vs_cov_toy_orig} we plot the overall system (expert and algorithm combined) AU-ROC for each desired coverage for the methods and in Figure \ref{fig:ap_vs_cov_toy_orig} we plot the overall system area under the precision-recall curve (AU-PR) versus the coverage; this is for the expert with $q=0.7$ and $p=1$. We can see that the curve for our method dominates the baselines over the entire coverage range for both AU-ROC and AU-PR, moreover the curves are concave and we can achieve higher performance by combining expert and algorithm than using both separately. Our method is able to achieve a higher maximum AU-ROC and AU-PR than both baselines: the difference between the maximum attainable AU-ROC of our method and Confidence is 0.043, 0.029, 0.016, 0.022 and 0.025 respectively for each of the five tasks.
There is a clear hierarchy between the 3 compared methods: our method dominates Confidence and Confidence in turn dominates ModelConfidence, in fact ModelConfidence is a special case of the Confidence baseline, since the expert does not have uniform performance over the domain there are clear gains in modeling the expert. 

This hierarchy continues to hold as we change the expert behavior as we vary the probabilities $p$ and $q$, in Table \ref{table:toy_expert} we show for each of the 5 tasks the difference between the average AU-ROC across all coverages (average value of the curves shown in Figure \ref{fig:auc_vs_cov_toy_orig}) for our method and the Confidence baseline for different expert probabilities and the difference between the maximum achievable AU-ROC. A positive average difference serves to show the degree of dominance of our method over the Confidence baseline, note that the difference alone cannot imply dominance of the curves however dominance is still observed. Our method improves on the baselines as the difference between $q$ and $p$ increases, this difference encodes the non-uniformity of the expert behavior over the domain.

\begin{figure}[ht]
\centering
\begin{subfigure}{.9\textwidth}
\centering

      \includegraphics[width=\textwidth,trim={5.0cm 0.0cm 5.0cm 0.0cm}]{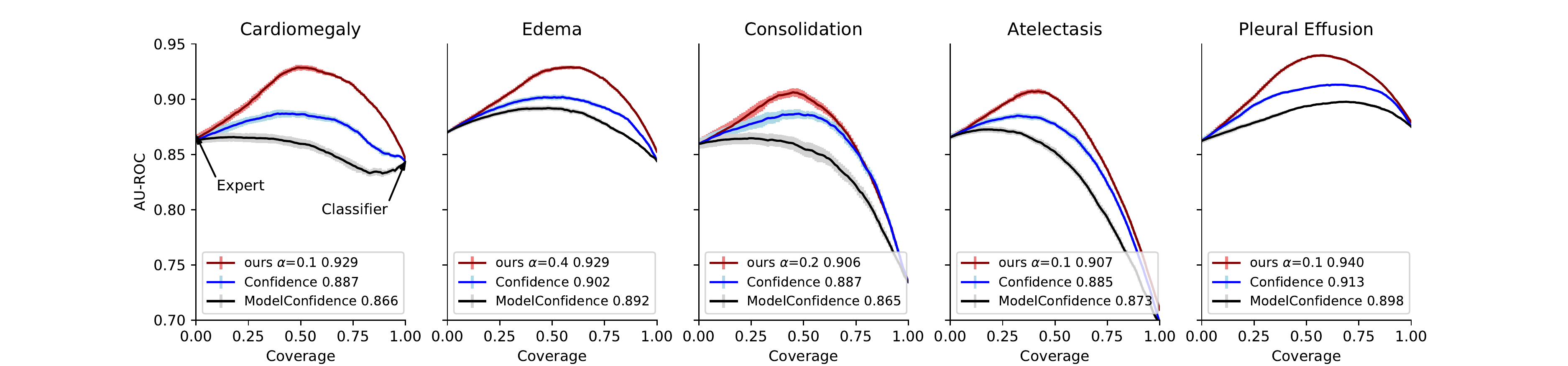}
      \caption{AU-ROC vs coverage for expert $q=0.7,p=1$, maximum AU-ROC is noted.}
      \label{fig:auc_vs_cov_toy_orig}
\end{subfigure}

\begin{subfigure}{.9\textwidth}
\centering
  \includegraphics[width=\textwidth,trim={5.0cm 0.0cm 5.0cm 0.0cm}]{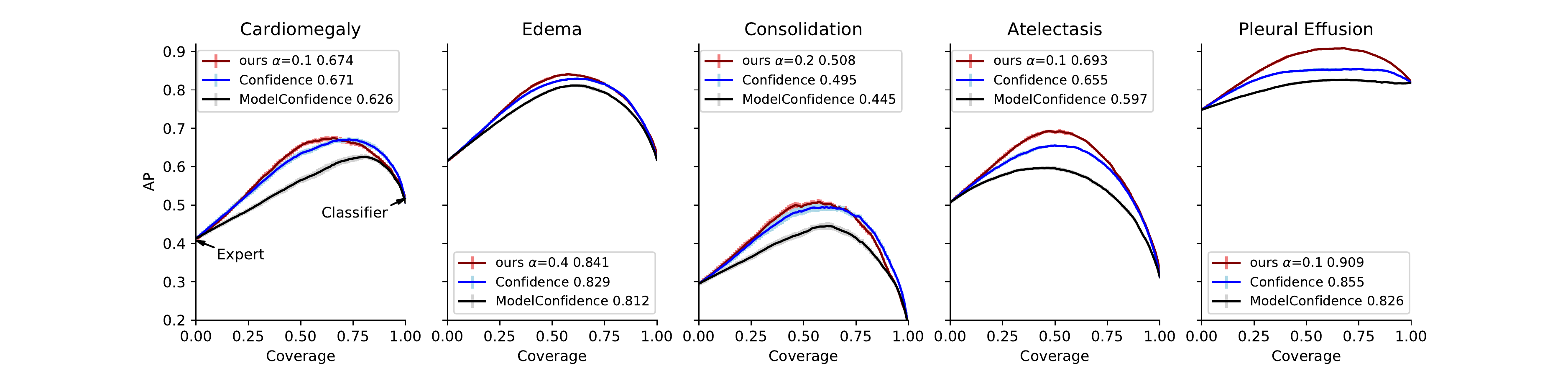}
\caption{AU-PR vs coverage for expert $q=0.7,p=1$, maximum AU-PR is noted.}
      \label{fig:ap_vs_cov_toy_orig}
\end{subfigure}

\caption{Plot of AU-ROC of the ROC curve (a) for each level of coverage (0 coverage means only the expert predicting and 1 coverage is only the classifier predicting) and of the area under the precision-recall curve (AU-PR) (b) for each of the 5 tasks comparing our method with the baselines on the training derived test set for the toy expert with $q=0.7,p=1$. We report the maximum AU-ROC and AU-PR achieved on each task, error bars are standard deviations derived from 10 runs (averaging over the expert's randomness).}
\label{fig:plots_auc_ap_toyexpert}
\end{figure}

\begin{table}[h]
\caption{Average difference in AU-ROC across all coverage and difference between maximum achievable AU-ROC between our method and the Confidence baseline for each of the 5 tasks and different toy expert probabilities $p$ and $q$; each entry is (average difference $\pm$ standard deviation; difference of maximums). The difference between our method and the ModelConfidence is roughly twice the values noted in table \ref{table:toy_expert}, only at Expert $(0.7,0.7)$ does Confidence and ModelConfidence achieve the same performance since the expert has uniform error over the domain.}
\label{table:toy_expert}
\vskip 0.15in
\begin{center}
\begin{small}
\begin{sc}
\resizebox{\textwidth}{!}{
\begin{tabular}{l|ccccc}
\toprule
Expert $(p,q)$ &  \textbf{Cardiomegaly} & \textbf{Edema} & \textbf{Consolidation}  &   \textbf{Atelectasis}  & \textbf{Pleural Effusion}  \\
\midrule
(0.5,0.7) &0.032$\pm$0.024; 0.002  &0.015$\pm$0.012; 0.007    &  0.015$\pm$0.008; 0.007    & 0.017$\pm$0.009; 0.007     & 0.007$\pm$0.003 ;0.007   \\ \hline
(0.5,0.9) & 0.032$\pm$0.017; 0.014 &  0.026$\pm$0.016; 0.024 & 0.010$\pm$0.005; 0.015    &   0.016$\pm$0.008; 0.026   & 0.012$\pm$0.010; 0.004  \\ \hline
(0.5,1)  & 0.022$\pm$0.012; 0.029 & 0.013$\pm$0.009; 0.019  & 0.007$\pm$0.008; 0.012    & 0.013$\pm$0.006; 0.020     & 0.010$\pm$0.008; 0.012   \\ \hline
(0.7,0.7) & 0.024$\pm$0.018; 0.005  & 0.011$\pm$0.009; 0.010  &  0.011$\pm$0.010; 0.009   &  0.006$\pm$0.006; 0.008    &  0.001$\pm$0.001; 0.003  \\ \hline 
(0.7,0.9) & 0.032$\pm$0.020; 0.024   & 0.010$\pm$0.007; 0.010  & 0.007$\pm$0.007; 0.017    &  0.014$\pm$0.008; 0.017    &0.010$\pm$0.006; 0.006   \\ \hline 
(0.7,1) &0.027$\pm$0.014; 0.042  &  0.016$\pm$0.010; 0.027 &  0.007$\pm$0.007; 0.019   &  0.013$\pm$0.007; 0.022    &  0.014$\pm$0.010; 0.027  \\ \hline 
(0.8,1) &  0.017$\pm$0.009; 0.023   & 0.011$\pm$0.008; 0.012   & 0.001$\pm$0.004; 0.007    & 0.012$\pm$0.006; 0.009     & 0.010$\pm$0.006; 0.018 \\ 
\bottomrule
\end{tabular}
}
\end{sc}
\end{small}
\end{center}
\vskip -0.1in
\end{table}

\subsubsection{Further Analysis}

\paragraph{Sample Complexity}
Training data for chest X-rays is a valuable resource that may not be abundantly available when trying to deploy a machine learning model in a new clinical setting where for example the imaging mechanism may differ. It is important to see the effectiveness of the proposed approaches when training data size is limited, this furthermore helps us understand the comparative sample complexity of our method versus the baselines.

\textbf{Experimental details.} We restrict the training data size for our model and baselines while keeping the same validation and testing data as previously; the validation data is used only for calibration of models and optimizing over choice of $\alpha$. We train using the same procedure as before and report the maximum achievable AU-PR and AU-ROC. The expert we defer to is the synthetic expert described above with $q=0.7$ and $p=1$.

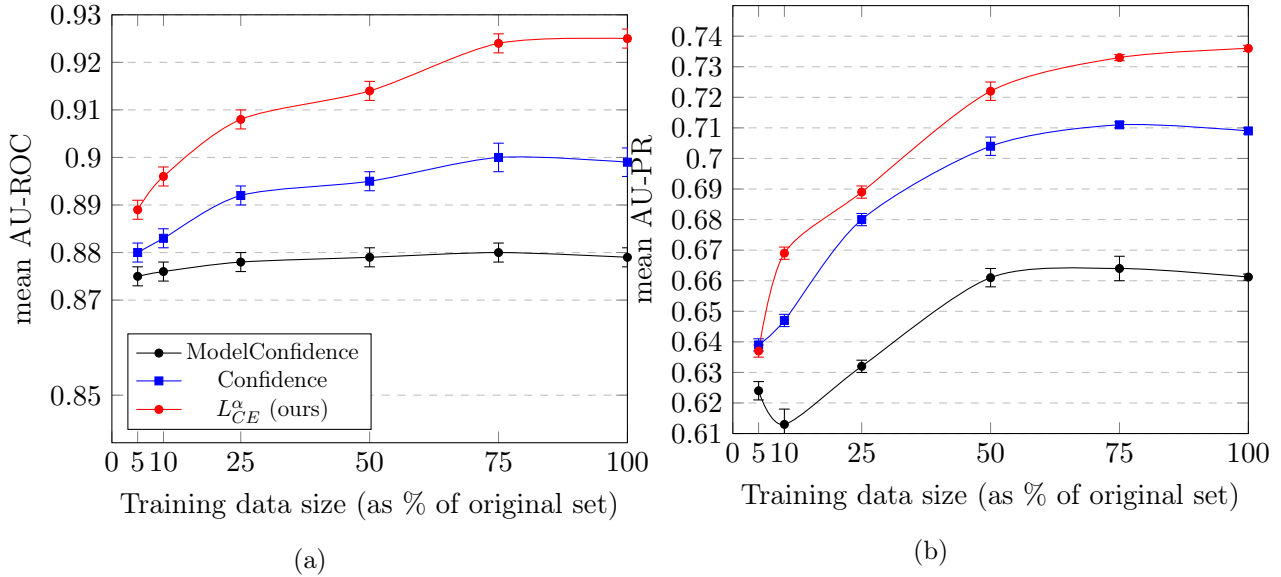
\begin{figure}[h]
\centering
\begin{subfigure}{.5\textwidth}
\centering
\begin{tikzpicture}
\begin{axis}[
    title={},
    xlabel={Training data size (as \% of original set)},
    ylabel={mean AU-ROC},
    xmin=0, xmax=100,
    ymin=0.84, ymax=0.93,
    xtick={0,5,10,25,50,75,100},
  ytick={0.85,0.87,0.88,0.89,0.9,0.91,0.92,0.93,0.95},
    legend pos=south west,
legend style={nodes={scale=0.8, transform shape}},
    ymajorgrids=true,
    grid style=dashed,
]

 \addplot+[
  black, mark options={black, scale=0.75},
  smooth, 
  error bars/.cd, 
    y fixed,
    y dir=both, 
    y explicit
] table [x=x, y=y,y error=error, col sep=comma] {
    x,  y,       error
    5,  0.875,       0.002
    10, 0.876, 0.002
    25, 0.878, 0.002
    50, 0.879, 0.002
    75, 0.88, 0.002
    100, 0.879, 0.002

};
\addlegendentry{ModelConfidence}

 \addplot+[
  blue, mark options={blue, scale=0.75},
  smooth, 
  error bars/.cd, 
    y fixed,
    y dir=both, 
    y explicit
] table [x=x, y=y,y error=error, col sep=comma] {
    x,  y,       error
    5,  0.880,       0.002
    10,  0.883,       0.002
    25,  0.892,       0.002
    50,  0.895,       0.002
    75,  0.9,       0.003
    100,  0.899,       0.003

};
\addlegendentry{Confidence}

 \addplot+[
  red, mark options={red, scale=0.75},
  smooth, 
  error bars/.cd, 
    y fixed,
    y dir=both, 
    y explicit
] table [x=x, y=y,y error=error, col sep=comma] {
    x,  y,       error
    5,  0.889,       0.002
    10,  0.896,       0.002
    25,  0.908,       0.002
    50,  0.914,       0.002
    75,  0.924,       0.002
    100,  0.925,       0.002

};
\addlegendentry{$L_{CE}^\alpha$ (ours)}

\end{axis}
\end{tikzpicture}

\caption{
}
\label{fig:samplecmplx_auc}
\end{subfigure}%
\begin{subfigure}{.5\textwidth}
\centering
\begin{tikzpicture}
\begin{axis}[
       title={},
    xlabel={Training data size (as \% of original set)},
    ylabel={mean AU-PR},
    xmin=0, xmax=100,
    ymin=0.61, ymax=0.75,
    xtick={0,5,10,25,50,75,100},
  ytick={0.61,0.62,0.63,0.64,0.65,0.66,0.67,0.68,0.69,0.7,0.71,0.72,0.73,0.74},
    legend pos=north west,
legend style={nodes={scale=0.8, transform shape}},
    ymajorgrids=true,
    grid style=dashed,
]

 \addplot+[
  black, mark options={black, scale=0.75},
  smooth, 
  error bars/.cd, 
    y fixed,
    y dir=both, 
    y explicit
] table [x=x, y=y,y error=error, col sep=comma] {
    x,  y,       error
    5,  0.624,       0.003
    10,  0.613,       0.005
    25,  0.632,       0.002
    50,  0.661,       0.003
    75,  0.664,       0.004
    100,  0.6612,       0.001
};
\addlegendentry{ModelConfidence}

 \addplot+[
  blue, mark options={blue, scale=0.75},
  smooth, 
  error bars/.cd, 
    y fixed,
    y dir=both, 
    y explicit
] table [x=x, y=y,y error=error, col sep=comma] {
    x,  y,       error
    5,  0.639,       0.002
    10,  0.647,       0.002
    25,  0.680,       0.002
    50,  0.704,       0.003
    75,  0.711,       0.001
    100,  0.709,       0.001
};
\addlegendentry{Confidence}

 \addplot+[
  red, mark options={red, scale=0.75},
  smooth, 
  error bars/.cd, 
    y fixed,
    y dir=both, 
    y explicit
] table [x=x, y=y,y error=error, col sep=comma] {
    x,  y,       error
    5,  0.637,       0.002
    10,  0.669,       0.002
    25,  0.689,       0.002
    50,  0.722,       0.003
    75,  0.733,       0.001
    100,  0.736,       0.001
};
\addlegendentry{$L_{CE}^\alpha$ (ours)}

   \legend{};
\end{axis}
\end{tikzpicture}

\caption{
}
\label{fig:samplecmplx_pr}
\end{subfigure}
\caption{Left figure shows the average of the maximum achievable AU-ROC for the 5 tasks (average over the tasks) when the changing the size of the training data (as a \% of the original set) and right figure shows the same for AU-PR }
\label{fig:samplecmplx_chexpert}
\end{figure}

\textbf{Results.} In Figure \ref{fig:samplecmplx_chexpert} we plot the average of the maximum achievable AU-ROC \ref{fig:samplecmplx_auc} and AU-PR \ref{fig:samplecmplx_pr} across the 5 tasks for the different methods as we vary the the training set size. We observe that our method consistently outperforms the baselines and continues to take advantage of further data as the baselines performance starts to saturate. If we look at the AU-ROC and AU-PR of the expert on deferred examples, we observe negligible differences as the training set size increases for each method, however if we look at classifier performance on the non-deferred examples, we start to observe a significant difference in AU-ROC and AU-PR for our method while the baselines lag behind. In Figure \ref{fig:plots_classauc_ap_toyexpert} (found in Appendix \ref{apx:exp_chexpert}) we plot the classifier AU-ROC on non-deferred examples versus the coverage for each of the 5 tasks, we can see for example on Cardiomegaly, our method at full training data obtains an AU-ROC that is at least 0.2 points greater than that of ModelConfidence at coverage levels less than 50\%. One expects  ModelConfidence to achieve the best performance when looking at non-deferred examples, and this in fact is true when we look at accuracy, however for AU-ROC, what happens is that the ModelConfidence baseline never defers on negative predicted examples due to the class imbalance which makes the model very confident in it's negative predictions. Thus, any positive labeled example that the model mistakenly labels as negative with high confidence will cause the AU-ROC to be reduced at low coverage levels. This also allows us to see that our method, and to an extent the Confidence baseline, make very different deferral decisions that factor in the expert.

\paragraph{Impact of input noise} In our previous experimental setup, the input domain of the classifier $\mathcal{X}$, the chest X-ray, is assumed to be sufficient to perfectly predict the label $Y$ as our golden standard is the prediction of expert radiologists from just looking at the X-ray. Therefore, given enough training data and a sufficiently rich model class, a learned classifier from $\mathcal{X}$ will be able to perfectly predict the target and won't need to defer to any expert to achieve better performance. In this set of experiments, we perform two studies: the first we hide the left part of the chest X-ray on both training and testing examples to obtain a new input domain $\tilde{\mathcal{X}}$. This now limits the power of any learned classifier even in the infinite data regime as the left part of the X-ray may hide crucial parts of the input. 
Figure \ref{fig:noisy_xray} shows this noise applied to a patient's X-ray, the size of the rectangular region was chosen to cover one side of the chest area, we later experiment with varying the scale of the noise.
In the second experiment, we train with the original chest X-rays but evaluate with noisy X-rays with noise in the form of erasing a randomly placed rectangular region of the X-ray. This second experiment is meant to the illustrate the robustness of the different methods to input noise. 

\begin{figure}
\centering
\begin{subfigure}{.5\textwidth}
\centering
\includegraphics[scale=0.7,trim={10.0cm 1.0cm 10.0cm 1.0cm}]{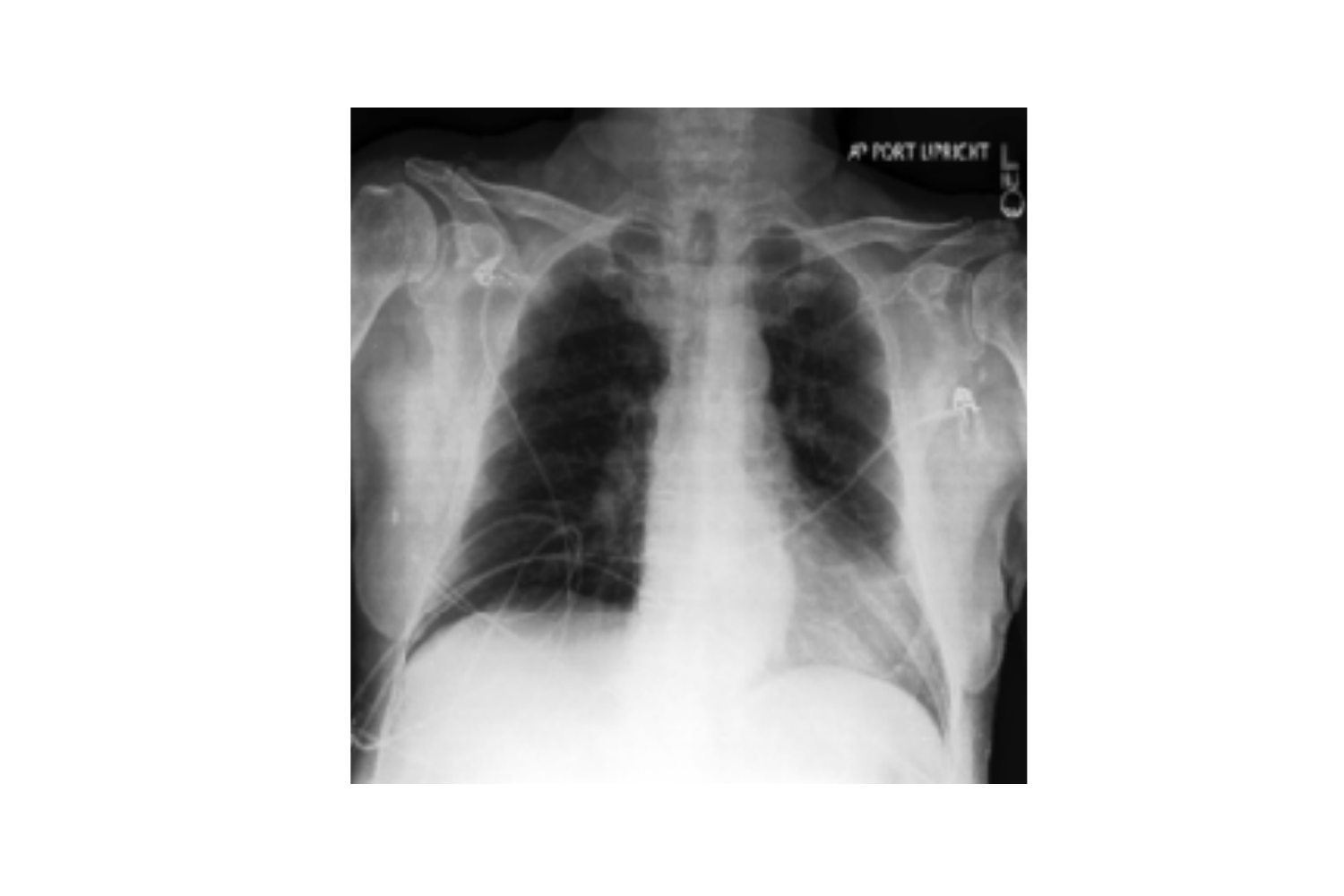}
\caption{Original X-ray
}
\label{fig:kvssystem}
\end{subfigure}%
\begin{subfigure}{.5\textwidth}
\centering
\includegraphics[scale=0.7,trim={10.0cm 1.0cm 10.0cm 1.0cm}]{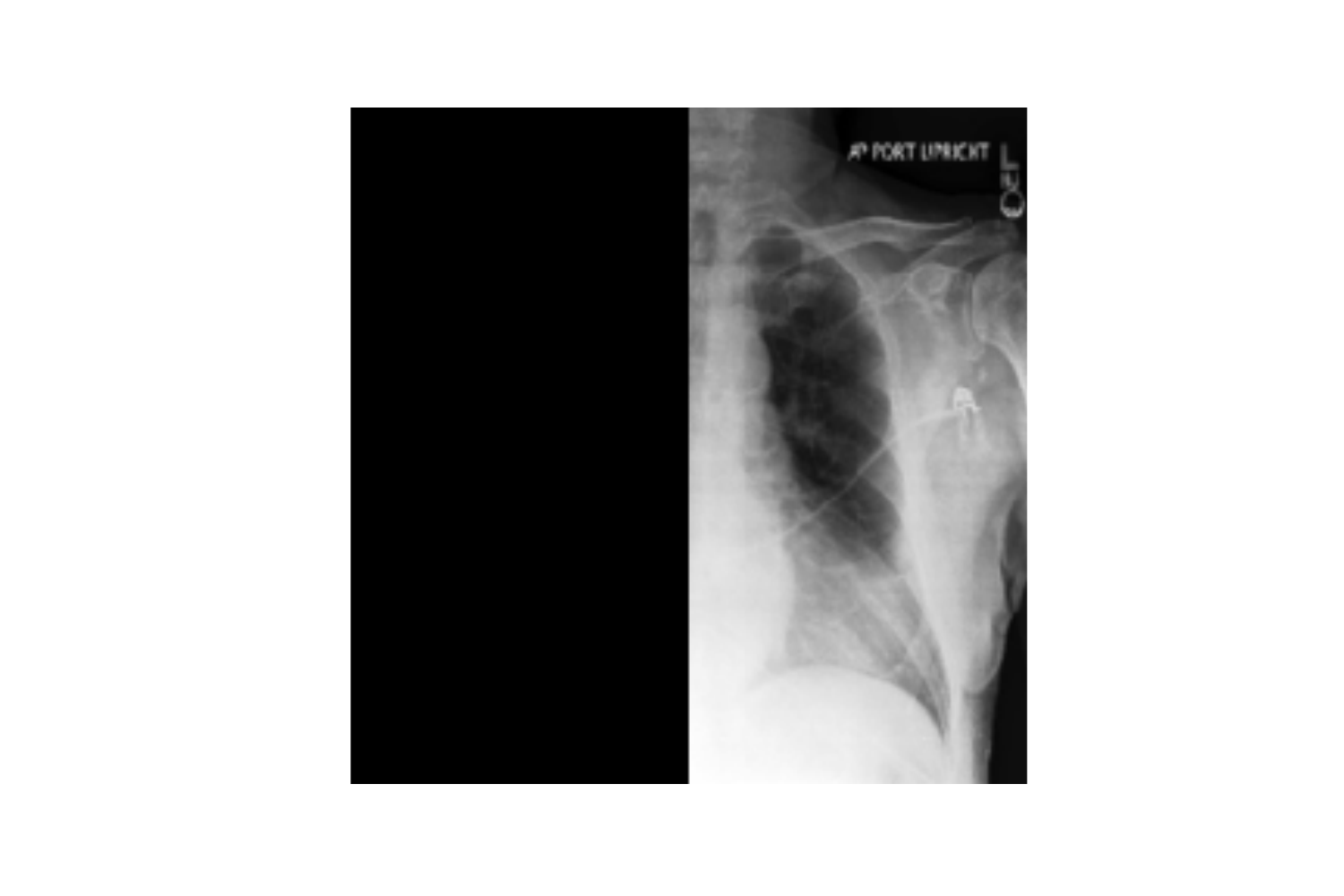}

\caption{X-ray with hidden left part
}
\label{fig:covvsacc_hidden}
\end{subfigure}
\caption{Left figure (a) shows the chest X-ray of a patient with Cardiomegaly, the right figure (b) shows that same X-ray but now with the left part hidden which is used as input to the models.  }
\label{fig:noisy_xray}
\end{figure}

\textbf{Experimental details.} The noise in the second set of experiments consists of a 2:1 (height:width) randomly located rectangular region of scale (area) that we vary from $0.1$ to $0.66$. The expert is the synthetic expert model with $q=0.7$ and $p=1$.

\textbf{Results.} In Figure \ref{fig:plots_auc_ap_toyexpert_left} we plot the AU-ROC and AU-PR of the different methods as we vary coverage when training and testing while hiding the left section of the X-rays. We can first observe that the maximum achievable performance for the different methods is significantly reduced, however the gap between the different methods is still observed. In Figure \ref{fig:noise_chexpert} we plot the 
average maximum AU-ROC and AU-PR across the 5 tasks as we vary the area of the rectangular region. While the performance of all the methods degrade with the scale of the noise, the gap between the methods remains constant in terms of AU-PR but diminishes in terms of AU-ROC as the performance of the baselines remains steady.

\begin{figure}[h]
\centering
\begin{subfigure}{.5\textwidth}
\centering
\begin{tikzpicture}
\begin{axis}[
    title={},
    xlabel={Noise scale},
    ylabel={mean AU-ROC},
    xmin=-0.1, xmax=0.7,
    ymin=0.84, ymax=0.93,
    xtick={0,0.1,0.2,0.33,0.5,0.66},
  ytick={0.85,0.87,0.88,0.89,0.9,0.91,0.92,0.93,0.95},
    legend pos=south west,
legend style={nodes={scale=0.8, transform shape}},
    ymajorgrids=true,
    grid style=dashed,
]

 \addplot+[
  black, mark options={black, scale=0.75},
  smooth, 
  error bars/.cd, 
    y fixed,
    y dir=both, 
    y explicit
] table [x=x, y=y,y error=error, col sep=comma] {
    x,  y,       error
    0,  0.879,       0.002
    0.1, 0.878, 0.002
    0.2, 0.873, 0.002
    0.33, 0.870, 0.002
    0.5, 0.868, 0.002
    0.66, 0.868, 0.002

};
\addlegendentry{ModelConfidence}

 \addplot+[
  blue, mark options={blue, scale=0.75},
  smooth, 
  error bars/.cd, 
    y fixed,
    y dir=both, 
    y explicit
] table [x=x, y=y,y error=error, col sep=comma] {
    x,  y,       error
    0,  0.896,       0.002
    0.1, 0.891, 0.002
    0.2, 0.884, 0.002
    0.33, 0.877, 0.002
    0.5, 0.871, 0.002
    0.66, 0.871, 0.002
};
\addlegendentry{Confidence}

 \addplot+[
  red, mark options={red, scale=0.75},
  smooth, 
  error bars/.cd, 
    y fixed,
    y dir=both, 
    y explicit
] table [x=x, y=y,y error=error, col sep=comma] {
    x,  y,       error
    0,  0.922,       0.002
    0.1, 0.914, 0.002
    0.2, 0.903, 0.002
    0.33, 0.893, 0.002
    0.5, 0.881, 0.002
    0.66, 0.875, 0.002

};
\addlegendentry{$L_{CE}^\alpha$ (ours)}

\end{axis}
\end{tikzpicture}

\caption{
}
\label{fig:noise_auc}
\end{subfigure}%
\begin{subfigure}{.5\textwidth}
\centering
\begin{tikzpicture}
\begin{axis}[
       title={},
    xlabel={Noise scale},
    ylabel={mean AU-PR},
    xmin=-0.1, xmax=0.7,
    ymin=0.55, ymax=0.75,
    xtick={0,0.1,0.2,0.33,0.5,0.66},
  ytick={0.55,0.57,0.59,0.61,0.63,0.65,0.67,0.69,0.71,0.73,0.74},
    legend pos=north west,
legend style={nodes={scale=0.8, transform shape}},
    ymajorgrids=true,
    grid style=dashed,
]

 \addplot+[
  black, mark options={black, scale=0.75},
  smooth, 
  error bars/.cd, 
    y fixed,
    y dir=both, 
    y explicit
] table [x=x, y=y,y error=error, col sep=comma] {
    x,  y,       error
    0,  0.6612,       0.002
    0.1, 0.648, 0.002
    0.2, 0.612, 0.002
    0.33, 0.595, 0.002
    0.5, 0.572, 0.002
    0.66, 0.560, 0.002
};
\addlegendentry{ModelConfidence}

 \addplot+[
  blue, mark options={blue, scale=0.75},
  smooth, 
  error bars/.cd, 
    y fixed,
    y dir=both, 
    y explicit
] table [x=x, y=y,y error=error, col sep=comma] {
    x,  y,       error
    0,  0.701,       0.002
    0.1, 0.682, 0.002
    0.2, 0.656, 0.002
    0.33, 0.619, 0.002
    0.5, 0.589, 0.002
    0.66, 0.566, 0.002
};
\addlegendentry{Confidence}

 \addplot+[
  red, mark options={red, scale=0.75},
  smooth, 
  error bars/.cd, 
    y fixed,
    y dir=both, 
    y explicit
] table [x=x, y=y,y error=error, col sep=comma] {
    x,  y,       error
    0,  0.736,       0.002
    0.1, 0.702, 0.002
    0.2, 0.678, 0.002
    0.33, 0.654, 0.002
    0.5, 0.616, 0.002
    0.66, 0.580, 0.002
};
\addlegendentry{$L_{CE}^\alpha$ (ours)}

   \legend{};
\end{axis}
\end{tikzpicture}

\caption{
}
\label{fig:noise_pr}
\end{subfigure}
\caption{Left figure shows the average of the maximum achievable AU-ROC for the 5 tasks (average over the tasks) when the changing the scale of the noise (size of rectangular region) and right figure shows the same for AU-PR. }
\label{fig:noise_chexpert}
\end{figure}
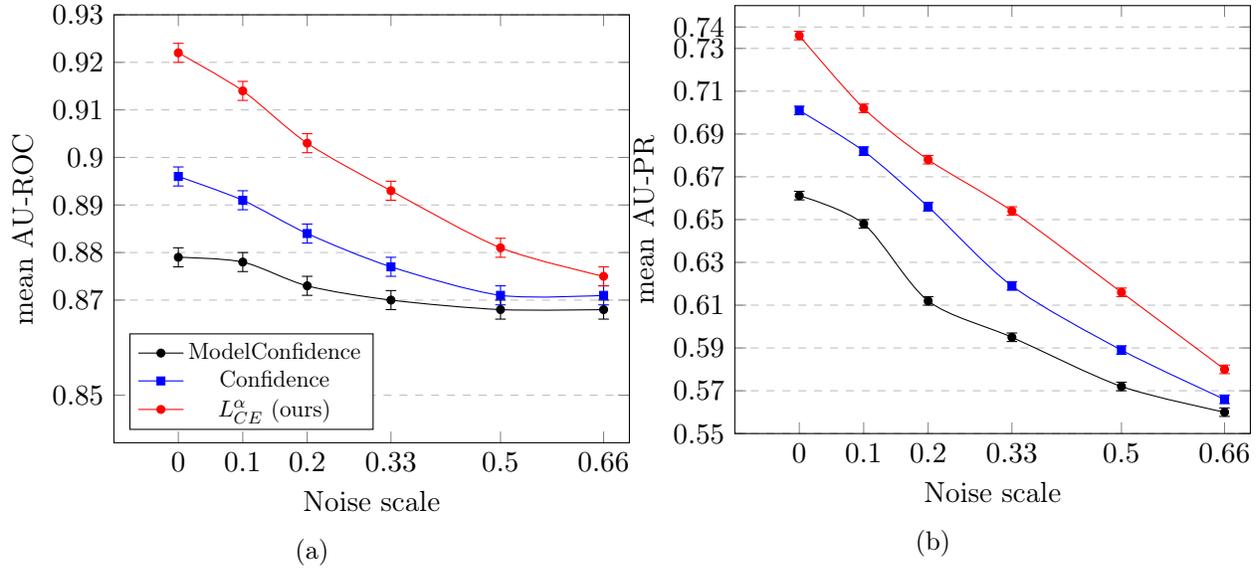

\begin{figure}[H]
\centering
\begin{subfigure}{.9\textwidth}
\centering

      \includegraphics[width=\textwidth,trim={5.0cm 0.0cm 5.0cm 0.0cm}]{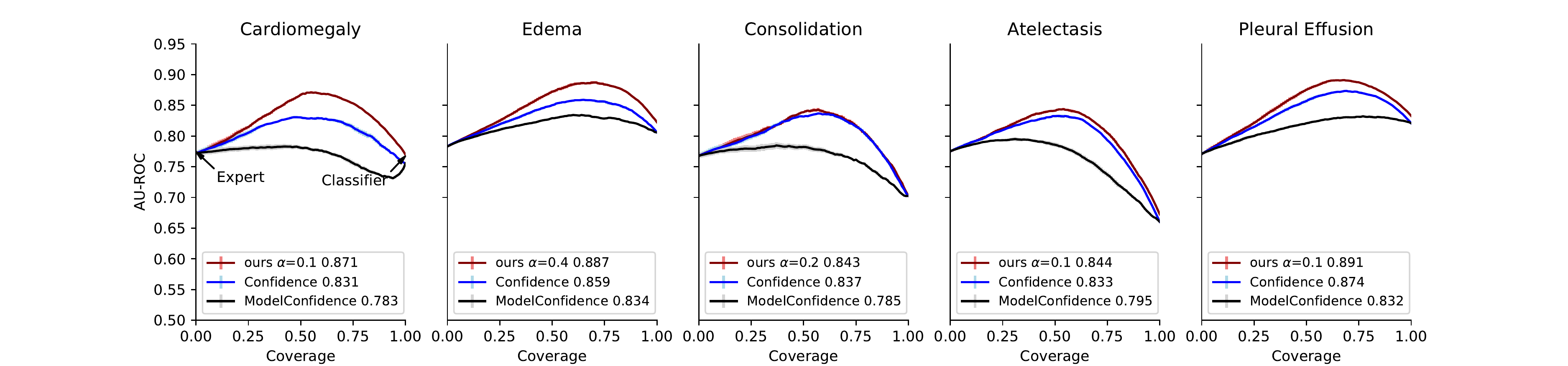}
      \caption{AU-ROC vs coverage when hiding left part of X-ray.}
      \label{fig:auc_vs_cov_toy_left}
\end{subfigure}

\begin{subfigure}{.9\textwidth}
\centering
  \includegraphics[width=\textwidth,trim={5.0cm 0.0cm 5.0cm 0.0cm}]{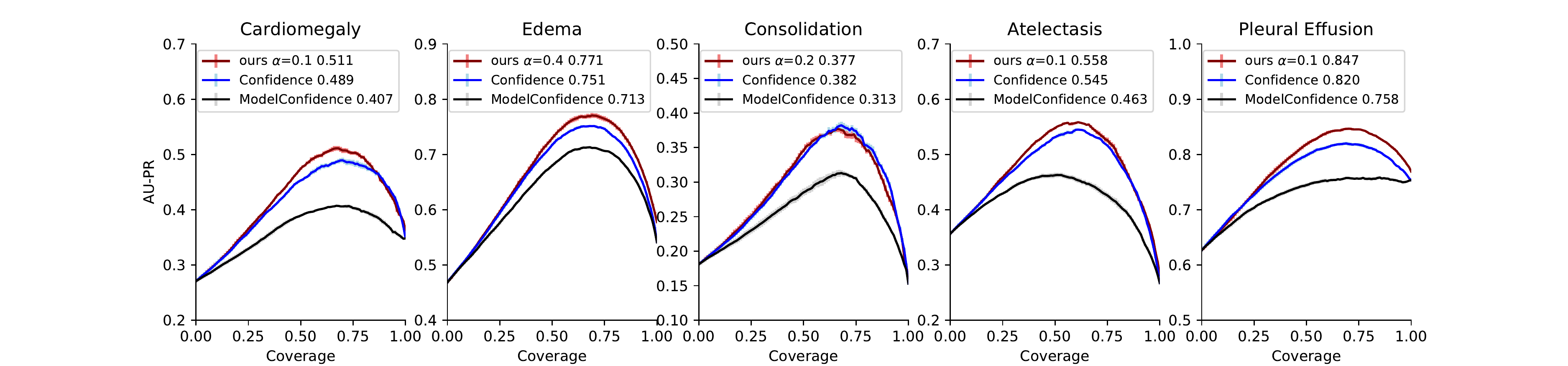}
\caption{AU-PR vs coverage when hiding left part of X-ray.}
      \label{fig:ap_vs_cov_toy_left}
\end{subfigure}

\caption{Plots of AU-ROC and AU-PR as we vary coverage when training and testing with chest X-rays that have their left section hidden. The expert model is $q=0.7$ and $p=1$.}
\label{fig:plots_auc_ap_toyexpert_left}
\end{figure}

\section{Conclusion}
In this work we explored a framework where the learning model can choose to defer to an expert or predict. We analyzed the framework theoretically and proposed a novel surrogate loss via a reduction to multiclass cost sensitive learning. Through experiments on image and text classification tasks, we showcased that our approach not only achieves better accuracy than confidence score baselines but does so with better sample complexity and computational cost. We hope that our method will inspire machine learning practitioners to integrate downstream decision makers into their learning algorithms. Future work will explore how to defer in settings where we have limited expert data,  learning from biased expert data and learning to defer to multiple experts simultaneously.

\section*{Acknowledgements}
This work was supported by NSF CAREER award $\#1350965$.

\bibliographystyle{alpha}
\bibliography{ref}
\newpage
\appendix

\section{Practitioner's guide to our approach}\label{apx:guide}
\subsection{General implementation}
Given a dataset of tuples $S = \{(x_i,y_i,m_i) \}_{i=1}^n$ where $x_i$ represents the covariates, $y_i$ is the target and $m_i$ are the expert labels, we want to  construct a classifier $h: \mathcal{X} \to \mathcal{Y}$ and rejector function $r: \mathcal{X} \to \{-1,1 \}$. Our method for predicting on a new example $x \in \mathcal{X}$ given expert context $z \in \mathcal{Z}$ that only the expert can observe, a function class $\mathcal{H}$ where $h \in \mathcal{H}:\mathcal{X} \to \mathbb{R}^{|\mathcal{Y}|+1}$ (an example would be the set of deep networks with $|\mathcal{Y}|+1$ output units) , and an expert $M: \mathcal{Z} \to \mathcal{Y}$ is summarized below in Algorithm \ref{alg:our_method}.

\begin{algorithm}[H]
	\DontPrintSemicolon 
	\SetAlgoLined
	\textbf{Input}: training data $S = \{(x_i,y_i,m_i) \}_{i=1}^n$, function class $\mathcal{H}$, example $x$, Expert $M$ and expert input $z$ \\
    $g_1, \cdots, g_{|\mathcal{Y}|},g_\bot \gets \arg \min_{\mathbf{g} \in \mathcal{H}} \sum_{i \in S} L_{CE}^\alpha(\mathbf{g},x_i,y_i,m_i)$ \\
     prediction $= 0$ \\
     $r(x) \gets \textrm{sign}(-\max_{y \in \mathcal{Y}}g_y(x) + g_\bot(x))$ \\
    \eIf{$r(x)=0$}{
      $h(x) \gets \arg\max_{y \in \mathcal{Y}} g_y(x)$ \\ 
      prediction  $\gets h(x)$
      }{
      $m \gets M(z)$ (expert query)\\
      prediction $\gets m$
      }
     \textbf{Return}: prediction
	\caption{Our proposed method for prediction on a new example $x \in \mathcal{X}$ with expert input $z \in \mathcal{Z}$}
	\label{alg:our_method}
\end{algorithm}
Where the loss $L_{CE}^\alpha$ used in algorithm is the following:
\begin{align*}
  L_{CE}^\alpha(h,r,x,y,m)=& -( \alpha \cdot \bI_{m = y} + \bI_{m \neq y} )\log\left(\frac{\exp(g_{y}(x))}{\sum_{y' \in \mathcal{Y} \cup \nonumber \bot}\exp(g_{y'}(x))} \right) \\&-  \bI_{m = y} \log\left(\frac{\exp(g_{\bot}(x))}{\sum_{y' \in \mathcal{Y} \cup \bot}\exp(g_{y'}(x))} \right) 
\end{align*}
Practically, integrating an expert decision maker into a machine learning model amounts to two modifications in training:
increasing the output size of the function class in consideration by an additional output unit  representing deferral and training with the loss $L_{CE}^\alpha$ instead of the cross entropy loss. We show how to implement $L_{CE}^\alpha$ in PyTorch below:
\begin{lstlisting}[language=Python]
def deferral_loss_L_CE(outputs, target, expert, k_classes, alpha):
    '''
    outputs: model outputs
    target: target labels
    expert: expert agreement labels for batch
    k_classes: cardinality of target Y
    '''
    batch_size = outputs.size()[0]
    defer_position = 
    outputs = torch.nn.functional.softmax(outputs, dim=1)
    loss = -expert*torch.log2(outputs[range(batch_size),k_classes]) 
           - (alpha*expert + (1-expert)) *
           torch.log2(outputs[range(batch_size), labels])  
    return torch.sum(loss)/batch_size
    
\end{lstlisting}

\subsection{Choice of $\alpha$}
The choice of the hyperparameter $\alpha$ has sizable influence on system performance. Naive validation over  $\alpha$ requires re-training on the entire training set from scratch over the search space. We find that a simple validation strategy often works as well as re-training from scratch especially in scenarios where there is little gain in adapting to the expert but there are major gains in being able to defer correctly. 

The strategy first requires splitting the training set into two sets $S_{T1}$ and $S_{T2}$ where $S_{T1}$ is larger than $S_{T2}$ (e.g. an 80-20 split), access to a validation set $S_V$ and a set of possible values $\mathcal{A}$ for $\alpha$ (an evenly spaced grid over $[0,10]$ is more than sufficient). The strategy then proceeds in two steps:
\begin{itemize}
    \item \textbf{Step 1:}Train on $S_{T1}$ with $L_{CE}^1$ (i.e. setting $\alpha =1$) to maximize system performance on $S_V$. One may find more success instead training on $S_{T1}$ with the \emph{cross-entropy loss} (however with the model having an extra output) to maximize \emph{classifier} performance on $S_V$ rather than system performance. Call the resulting model of this first step $M_1$
    \item \textbf{Step 2:} For each $\alpha \in \mathcal{A}$, fine-tune on $S_{T2}$ starting from model $M_1$ to maximize system performance on $S_V$ measuring it with the rejector $r(x) = \mathbb{I}\{-\max_{y \in \mathcal{Y}}g_y(x) + g_\bot(x) \geq \tau\} $ where the threshold $\tau$ is chosen to maximize performance on $S_V$ post-hoc. The resulting model $M_1'$ and $\tau^*$ that obtains best system performance across all choices of $\alpha$ and choices of $\tau$ is the final model. \\
    \item \textbf{Inference time:} Use the rejector defined by $r(x) = \mathbb{I}\{-\max_{y \in \mathcal{Y}}g_y(x) + g_\bot(x) \geq \tau^* \} $ and proceed as in Algorithm \ref{alg:our_method}.
\end{itemize}

\textit{Note} that system performance here refers to metrics measured with respect to the machine+expert system with deferral while classifier performance refers to metrics measured as if the system never deferred. 
\section{Experimental Details and Results}\label{apx:experiments}
All experiments were run on a Linux system with an NVIDIA Tesla K80 GPU on PyTorch 1.4.0.

\subsection{CIFAR-10}\label{apx:cifar10}
\textbf{Implementation Details.}
We employ the implementation in \url{https://github.com/xternalz/WideResNet-pytorch} for the Wide Residual Networks. To train, we run SGD with an initial learning rate of 0.1, Nesterov momentum at 0.9 and weight decay of 5e-4 with a cosine annealing learning rate schedule \cite{loshchilov2016sgdr}. We train for a total of 200 epochs for all experiments, at this point the network has perfectly fit the training set, we found that early stopping based on a validation set did not make any difference and similarly training for more than 200 epochs also did not hurt test accuracy. 

\textbf{Expert Accuracy.} In Table \ref{table:cifar10-expert-acc} we show the accuracy of the expert on the deferred examples versus the classes the expert can predict $k$. We can see that our method $L_{CE}^{.5}$ has higher expert accuracy than all other baselines except at $k=1,2$ where coverage is very high. This contrasts with Figure \ref{fig:covvsacc} that shows the classifier accuracy on non-deferred accuracy where $L_{CE}^{.5}$ had lower accuracy for each expert level compared to Confidence and $L_{CE}^1$. Hence there is a clear trade-off between choosing the hyper-parameter $\alpha <1$ and $\alpha=1$. For $\alpha<1$, the model will prefer to always defer to the expert if it is correct, this is advantageous in this setup as the expert is perfect on a subset of the data and uniformly random on the other. However, for $\alpha=1$, the model will compare the confidence of the expert and the model essentially performing the computation of the Bayes rejector $r^B$ as shown by the consistency of the loss $L_{CE}^1$; note that for $\alpha \neq 1$ the loss $L_{CE}$ is no longer consistent.

\begin{table}[H]
\caption{Accuracy of the expert on deferred examples shown for the methods and baselines proposed with varying expert competence (k) on CIFAR-10.}
\label{table:cifar10-expert-acc}
\vskip 0.15in
\begin{center}
\begin{small}
\begin{sc}
\resizebox{\textwidth}{!}{
\begin{tabular}{lcccccccccr}
\toprule
Method /\ Expert (k) & 1  & 2 & 3&4& 5&6&7&8&9&10  \\
\midrule
$L_{CE}^1$  & 73.65 & 86.01 &73.66 &87.41 &88.81 &94.7 &96.67 &98.72 &98.65 &100 \\
$L_{CE}^{.5}$  & 86.44 & 90.96 &\textbf{92.65} &\textbf{91.67} & \textbf{93.71} &\textbf{96.32} &\textbf{97.61} &98.77 &\textbf{99.24} &\textbf{100} \\
Confidence  & \textbf{87.5} & \textbf{92.74 }&88.88 &88.3 &92.8 & 94.56 &96.76 &\textbf{98.89} &98.89 &100 \\
OracleReject  & 85.3 & 90.49 &88.23 &91.13 &89.33 &93.61 &95.45& 96.82 &98.45 &100 \\
\bottomrule
\end{tabular}}
\end{sc}
\end{small}
\end{center}
\vskip -0.1in
\end{table}

\textbf{Increasing data size.} In table \ref{table:cifar10-increasing-data} we show the accuracy of the classifier and the coverage of the system for our method compared to the baseline Confidence for expert $k=5$. We can see that when data is limited, our method retains high classification accuracy for the classifier versus the baseline. This is due in fact to the low coverage of our method compared to Confidence, as data size grows the coverage our method increases as now the classifier's performance improves and the system can now safely defer to it more often. On the other hand, the baseline remains at almost constant coverage, not adapting to growing data sizes.

\begin{table}[H]
\caption{Accuracy of the classifier on non-deferred examples shown for our method $L_{CE}^1$ and baseline Confidence with varying training set size for expert $k=5$ on CIFAR-10.}
\label{table:cifar10-increasing-data}
\vskip 0.15in
\begin{center}
\begin{small}
\begin{sc}
\resizebox{\textwidth}{!}{
\begin{tabular}{lcccccccr}
\toprule
Method /\ Data size (thousands) & 1  & 2 & 3&5& 8&10&20&50  \\
\midrule
$L_{CE}^1$ (classifier) & \textbf{62.84} & \textbf{71.51} & \textbf{72.63} & \textbf{75.03 }& 80.1 &\textbf{82.11} & 86.44 & \textbf{95.42} \\
Confidence (classifier) & 50.31 & 59 & 66.3 &70.12 &\textbf{80.33 }&78.67 & \textbf{87.01} & 92.45 \\
\midrule
$L_{CE}^1$ (coverage) & 25.7 & 35.87 & 40.42 &49.62 &46.38 &46.51 & 50 &71.35 \\
Confidence (coverage) & \textbf{69.32} &\textbf{ 72.93} & \textbf{71.99} & \textbf{75.05} & \textbf{73.09} & \textbf{65.9 }& \textbf{74.16} & \textbf{72.12} \\
\bottomrule
\end{tabular}}
\end{sc}
\end{small}
\end{center}
\vskip -0.1in
\end{table}

\subsection{CIFAR-10H}
\textbf{Class-wise Accuracy of Expert.} Table \ref{table:cifar10h-expertacc} shows the average accuracy of the synthetic \texttt{CIFAR10H} \cite{peterson2019human} expert on each of the 10 classes. We can see that the expert has very different accuracies for the classes which gives an opportunity for an improvement.

\textbf{Results.} Table \ref{table:cifar10h-complete} shows full experimental results for the CIFAR-10H results.

\begin{table}[H]
\caption{Accuracy of the \texttt{CIFAR10H} \cite{peterson2019human} expert on each of the 10 classes }
\label{table:cifar10h-expertacc}
\vskip 0.15in
\begin{center}
\begin{small}
\begin{sc}
\resizebox{\textwidth}{!}{
\begin{tabular}{lcccccccccc}
\toprule
Class  & 1  & 2 & 3&4&5&6&7&8&9&10 \\
\midrule
Accuracy &95.15&97.23&94.75&91.58&90.51&94.90&96.22&97.91&97.33&96.74\\
\bottomrule
\end{tabular}}
\end{sc}
\end{small}
\end{center}
\vskip -0.1in
\end{table}

\begin{table}[H]
\caption{Complete results of table \ref{table:cifar10h} comparing our proposed approaches and baseline.}
\label{table:cifar10h-complete}
\vskip 0.15in
\begin{center}
\begin{small}
\begin{sc}

\begin{tabular}{lcccr}
\toprule
Method & System Accuracy  & Coverage & Classifier Accuracy & Expert Accuracy   \\
\midrule
$L_{CE}$ impute   & \textbf{96.29}$\pm$0.25 & 51.67$\pm$1.46 &\textbf{ 99.2} $\pm$ 0.08 &\textbf{93.18 }$\pm$ 0.48 \\
$L_{CE}$ 2-step   & 96.03$\pm$0.21 & 60.81$\pm$0.87 & 98.11 $\pm$ 0.22 &92.77 $\pm$ 0.58\\
Confidence    & 95.09$\pm$0.40 & \textbf{79.48}$\pm$5.93 & 96.09 $\pm$ 0.42 &90.94 $\pm$ 1.34\\
\bottomrule
\end{tabular}

\end{sc}
\end{small}
\end{center}
\vskip -0.1in
\end{table}
\subsection{CIFAR-100}\label{apx:expcifar100}
\textbf{Results.} In figure \ref{fig:cifar100-kvssystem} we plot the accuracy of the combined algorithm and expert system versus $k$, the number of classes the expert can predict. We can see that our method dominates the baseline over all k. In table \ref{table:cifar100-all-results} we show expert, classifier and system accuracy along with coverage of both methods. Our approach $L_{CE}^1$ obtains both better expert and classifier accuracy however gets lower coverage than Confidence.

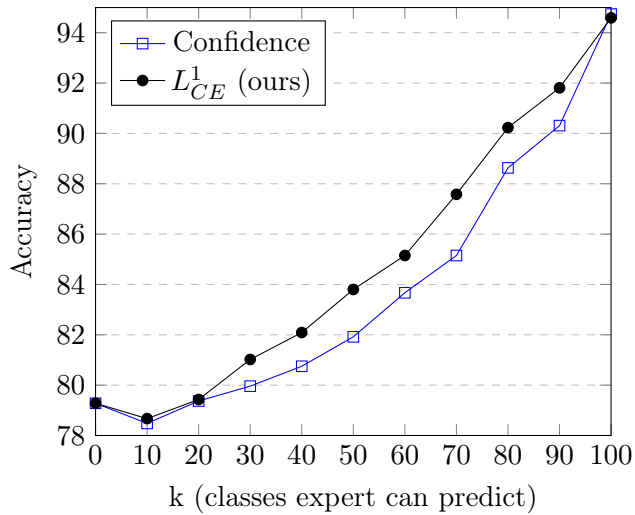
\begin{figure}[H]
\centering
\begin{tikzpicture}
\begin{axis}[
    title={},
    xlabel={k (classes expert can predict)},
    ylabel={Accuracy},
    xmin=0, xmax=100,
    ymin=78, ymax=95,
    xtick={0,10,20,30,40,50,60,70,80,90,100},
  ytick={78,80,82,84,86,88,90,92,94},
    legend pos=north west,
    ymajorgrids=true,
    grid style=dashed,
]
 
\addplot[
    color=blue,
    mark=square,
    ]
    coordinates {
    (0,79.28)(10,78.48)(20,79.37)(30,79.96)(40,80.75)(50,81.92)(60,83.67)(70,85.15)(80,88.63)(90,90.31)(100,94.74)
    };
    \addlegendentry{Confidence}
    
        \addplot[
    color=black,
    mark=*,
    ]
    coordinates {
    (0,79.28)(10,78.67)(20,79.43)(30,81.02)(40,82.09)(50,83.8)(60,85.15)(70,87.58)(80,90.23)(90,91.81)(100,94.59)
    };
    \addlegendentry{$L_{CE}^{1}$ (ours)}

\end{axis}
\end{tikzpicture}

\caption{Comparison of the developed method  $L_{CE}^1$ on CIFAR-100 versus the confidence baseline. k is the number of classes the expert can predict}
\label{fig:cifar100-kvssystem}
\end{figure}
\begin{table}[H]
\caption{Accuracy of the expert on deferred examples shown for the methods and baselines proposed with varying expert competence (k) on CIFAR-100.}
\label{table:cifar100-all-results}
\vskip 0.15in
\begin{center}
\begin{small}
\begin{sc}
\resizebox{\textwidth}{!}{
\begin{tabular}{lcccccccccr}
\toprule
Method /\ Expert (k) & 10  & 20 & 30&40& 50&60&70&80&90&100  \\
\midrule
$L_{CE}^1$ (system)  & \textbf{78.67} & \textbf{79.43} &\textbf{81.02} &\textbf{82.09 }&\textbf{83.8} &\textbf{85.15} &\textbf{87.58} &\textbf{90.23} &\textbf{91.81 }&94.59 \\
Confidence (system)  &78.48 & 79.37 & 79.67 &80.75 &81.92 & 83.67 & 85.15 & 88.63 & 90.31 & \textbf{94.74} \\
\midrule
$L_{CE}^1$ (coverage)  & 89.19 & 82.44 & 84.79 &71.66 & 74.52 & 65.72 & 62.23 &59.37 &52.15 &49.07 \\
Confidence (coverage) & \textbf{99.17} & \textbf{95.47}& \textbf{93.96} & \textbf{86.64} & \textbf{86.71 }& \textbf{80.67} & \textbf{79.56} & \textbf{75.36 }&\textbf{ 72.39} & \textbf{63.32} \\
\midrule
$L_{CE}^1$ (classifier) & \textbf{82.35} & \textbf{84.03} & \textbf{84.07} & \textbf{85.29} &\textbf{86.44} &\textbf{87.78} &\textbf{90.13} &\textbf{91.89} &\textbf{92.4} &94.59\\
Confidence (classifier) & 78.99& 80.66&81.79 &84.75 &84.62 & 87.30 &88.75 & 90.97 &92.07 &\textbf{94.97} \\
\midrule
$L_{CE}^1$ (expert) & \textbf{47.36} & \textbf{57.8} &\textbf{68.87} &\textbf{73.99} &\textbf{76.06} &\textbf{79.65} &\textbf{83.37} &\textbf{87.79} &\textbf{91.16} &\textbf{94.57} \\
Confidence (expert) & 18.07 & 52.09 & 51.49 &54.79 &64.4 & 68.55 & 71.13 & 82.11 &85.70 &94.30 \\
\bottomrule
\end{tabular}}
\end{sc}
\end{small}
\end{center}
\vskip -0.1in
\end{table}

\subsection{Hate Speech experiments}\label{apx:exphate}
\textbf{Implementation details.} We train all models with Adam for 15 epochs and select the best performing model on the validation set.

\textbf{Results.} Table \ref{table:hatespeech-all-results} shows complete results of our method, baselines, expert and classifier. The performance of our method and the baselines all achieve comparable results.

\begin{table}[H]
\caption{Detailed results for our method and baselines on the hate speech detection task \cite{davidson2017automated}. sys: system accuracy, class: classifier accuracy, disc: system discrimination, AAE-biased: Expert 2 that has higher error rate for AAE group, non-AAE biased: Expert 3 that has higher error for non AAE tweets }
\label{table:hatespeech-all-results}
\vskip 0.15in
\begin{small}
\begin{sc}
\resizebox{\textwidth}{!}{
\begin{tabular}{lccc|ccr}
\toprule
Method/Expert &  & Fair &  &  &  AAE-biased  &   \\
\midrule
 & sys & class & disc & sys & class & disc   \\
\midrule
$L_{CE}^1$ (ours) & 93.36 $\pm$ 0.16 & \textbf{95.60}  $\pm$ 0.44 & \textbf{0.294}  $\pm$0.03 & 92.91 $\pm$ 0.17 &\textbf{ 94.67} $\pm$ 0.61 & \textbf{0.37} $\pm$ 0.06  \\
Confidence & 93.22  $\pm$0.11 & 94.49 $\pm$ 0.12 & 0.45 $\pm$ 0.02 & 92.42 $\pm$ 0.40 & 94.56 $\pm$ 0.40 & 0.41 $\pm$ 0.02 \\
Oracle & \textbf{93.57}  $\pm$0.11 & 94.87  $\pm$0.22 & 0.32  $\pm$0.02 & \textbf{93.22}  $\pm$0.11 & 94.49  $\pm$0.12 & 0.449  $\pm$0.024 \\
\midrule
Expert & 89.76 & -- & 0.031 & 84.28 & -- & 0.071  \\
Classifier & 88.26 & 88.26 & 0.226 & 88.26 & 88.26 & 0.226\\
\bottomrule
\end{tabular}}
\end{sc}
\end{small}
\vskip 0.1in
\begin{small}
\begin{sc}
\begin{tabular}{lccr}
\toprule
Method/Expert  &  & non-AAE biased   & \\
\midrule
  & sys & class & disc \\
\midrule
$L_{CE}^1$ (ours)   & 90.42 $\pm$ 0.38 & \textbf{94.04}  $\pm$0.81 & 0.231  $\pm$0.04\\
Confidence  & 90.60  $\pm$0.13 & 93.68 v0.24 & 0.15  $\pm$0.03\\
Oracle  & \textbf{91.09} $\pm$ 0.12 &  92.57  $\pm$0.15 & \textbf{0.15}  $\pm$0.02\\
\midrule
Expert  & 80.4 & -- & 0.084\\
Classifier & 88.26 & 88.26 & 0.226\\
\bottomrule
\end{tabular}
\end{sc}
\end{small}
\vskip -0.1in
\end{table}

\subsection{Baseline Implementation}\label{apx:madras}
\textbf{Description of \cite{madras2018predict} approach.}
A different approach to our method, is to try directly to approximate the system loss \eqref{eq:system_loss_general},  this was the road taken by \cite{madras2018predict} in their differentiable model method. Let us introduce the loss used in \cite{madras2018predict}:
\begin{equation}
    L(h,r,M) = \bE_{(x,y) \sim \mathbf{P},m \sim M|(x,y)}\left[ (1-r(x,h(x))) l(y,h(x)) + r(x,h(x)) l(y,m)\right] \label{eq:madras_loss}
\end{equation}
where $h: \mathcal{X} \to \Delta^{|\mathcal{Y}|-1}$ (classifier), $r: \mathcal{X} \times \Delta^{|\mathcal{Y}|-1} \to \{0,1\}$ (rejector) and the expert $M: \mathcal{Z} \to \Delta^{|\mathcal{Y}|-1}$. 
\cite{madras2018predict} considers only binary labels and  uses the logistic loss for $l(.,.)$ and thus requires the expert to produce uncertainty estimates for it's predictions instead of only a label; we can extend this to the multiclass setting by using the cross entropy loss for $l$. It is clear that the loss \eqref{eq:madras_loss} is non-convex in $r$, hence to optimize 
it \cite{madras2018predict}  estimates the gradient through the
Concrete relaxation \cite{maddison2016concrete,jang2016categorical}. However, in the code of  \cite{madras2018predict} found at \url{https://github.com/dmadras/predict-responsibly}, the authors replace $r(x)$ by it's estimated probability from it's model. \cite{madras2018predict} considers an additional parameter $\gamma_{defer}$ found in the code, however it is not clear what effect this parameter has as we found it's description in  the paper did not match the code. In detail, let $r_0,r_1 : \mathcal{X} \to \mathbb{R}$ and $r(x) = \arg \max_{i \in \{0,1\}} r_i$, the loss \cite{madras2018predict} considers is:

\begin{equation}
    \tilde{L}(h,r,M) = \bE_{(x,y) \sim \mathbf{P},m \sim M|(x,y)}\left[  \frac{ \exp(r_0(x))}{\exp(r_0(x)) + \exp(r_1(x))} l(y,h(x)) + \frac{\exp(r_1(x))}{\exp(r_0(x)) + \exp(r_1(x))} l(y,m)\right] \label{eq:madras_loss_ours}
\end{equation}
All terms in loss \eqref{eq:madras_loss_ours} are on the same scale which is crucial for the model to train well. We explicitly have two functions $r_0$ and $r_1$ defining $r$ even though $r$ is binary; this is for ease of implementation.

Another key detail of \cite{madras2018predict} approach, is that the classifier is independently trained of the rejector by  stopping the gradient from $r$ to
backpropagate through $h$. This no longer allows $h$ to adapt to the expert, $h$ is trained with the cross entropy loss on it's own concurrently with $r$.

\textbf{CIFAR-10 details.} In our CIFAR-10 setup, the dataset $S$ contains only the final prediction $m$ of the expert $M$, thus to compute $l(y,m)$ we set $l(y,m) = - \log(1-\epsilon)$ if $y=m$ and $l(y,m) = - \log(\frac{1}{|\mathcal{Y}|})$ if $y \neq m$ (simulating a uniform prediction in accordance with our expert behavior) with $\epsilon = 10^{-12}$. One could instead train a network to model the expert's prediction, we found this approach to fail as there is a big amount of noise in the labels caused by the expert's random behavior.

\textbf{Results on CIFAR-10.} For expert $k<8$, we found that the \cite{madras2018predict} baseline to almost never defer to the expert and  when $k=8,9$ at the end of training (200 epochs) the rejector never defers but the optimal system is found in the middle of training ($\sim$100 epochs). The optimal systems achieve 46.27 and 40.22 coverage, 98.81 and 98.89 expert accuracy on deferred examples and 89.38 and 89.40 classifier accuracy on non-deferred examples respectively for $k=8,9$. The classifier alone for the optimal systems achieve $\sim$86 classification accuracy on all of the validation set for both experts, notice that there is not much difference between the classification accuracy on all the data and non-deferred examples, while for our method and other baselines there is a considerable increase. This indicates that the rejector is only looking at the expert loss and ignoring the classifier

What is causing this behavior is that as the classifier $h$ trains, it's loss $l(y,h(x))$ eventually goes to $0$, however the loss of the expert $l(y,m)$ is either $0$ or equal to $-\log(0.1)$, hence the rejector will make the easier decision to never defer. At initial epochs, we have a non-trivial rejector as the classifier $h$ is still learning, and the coverage progressively grows till $100\%$ over training. Essentially, what \cite{madras2018predict} approach is trying to do is choosing between the lower cost between expert and classifier: a cost-sensitive learning problem at it's heart. Therefore, one can use the losses developed here to tackle the problem better; we leave this to future investigations. Another potential fix is to learn the classifier and rejector on two different data sets.

\begin{table}[h]
\caption{System accuracy of our implementation of \cite{madras2018predict} and our method and baselines with varying expert competence (k) on CIFAR-10.}
\label{table:cifar10-madras}
\vskip 0.15in
\begin{center}
\begin{small}
\begin{sc}
\resizebox{\textwidth}{!}{
\begin{tabular}{lcccccccccr}
\toprule
Method /\ System accuracy (k) & 1  & 2 & 3&4& 5&6&7&8&9&10  \\
\midrule
$L_{CE}^{.5}$  & \textbf{90.92} & \textbf{91.01} &\textbf{91.94} &\textbf{92.69} &\textbf{93.66} &\textbf{96.03} &\textbf{97.11 }&\textbf{98.25} &\textbf{99} &\textbf{100} \\
$L_{CE}^{1}$  & 90.41 & 91.00 &91.47 &92.42 &93.4 &95.06 &96.49 &97.30 &97.70 &100 \\
Confidence  & 90.47 & 90.56 &90.71 &91.41 &92.52 &94.15 &95.5 &97.35 &98.05 &100 \\
OracleReject  & 89.54 & 89.51 &89.48 &90.75 &90.64 &93.25 &95.28& 96.52 &98.16 &100 \\
\cite{madras2018predict}  & 90.40 & 90.40 &90.40 &90.40 &90.40 &90.40 &90.40& 94.48 &95.09&100 \\
\bottomrule
\end{tabular}}
\end{sc}
\end{small}
\end{center}
\vskip -0.1in
\end{table}

\subsection{CheXpert Experiments}\label{apx:exp_chexpert}

\begin{figure}[h]
\centering
\begin{subfigure}{.9\textwidth}
\centering

      \includegraphics[width=\textwidth,trim={5.0cm 0.0cm 5.0cm 0.0cm}]{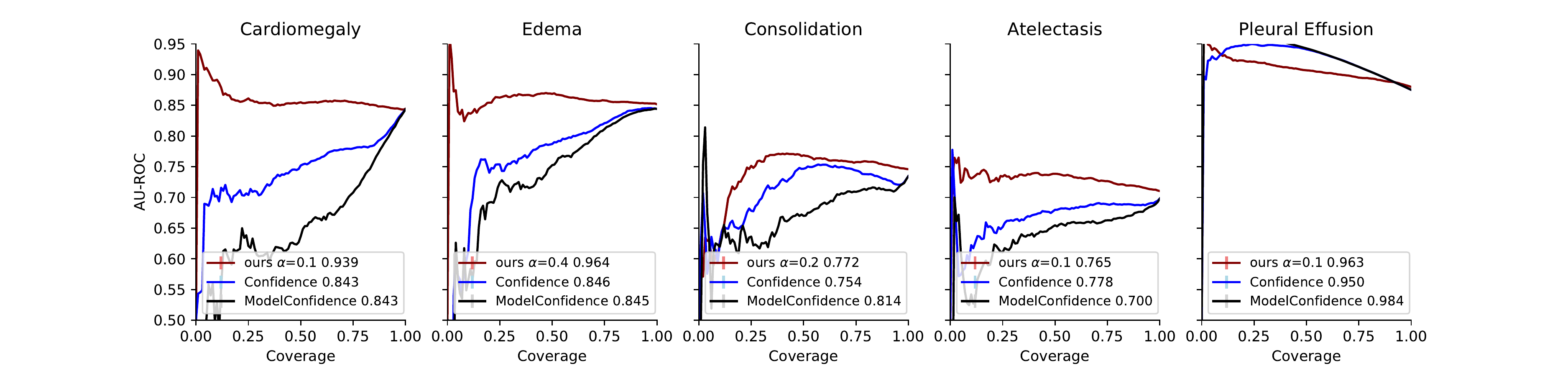}
      \caption{classifier AU-ROC on non-deferred examples vs coverage for expert $q=0.7,p=1$ with 100\% of training data.}
      \label{fig:classauc_vs_cov_toy}
\end{subfigure}

\begin{subfigure}{.9\textwidth}
\centering
  \includegraphics[width=\textwidth,trim={5.0cm 0.0cm 5.0cm 0.0cm}]{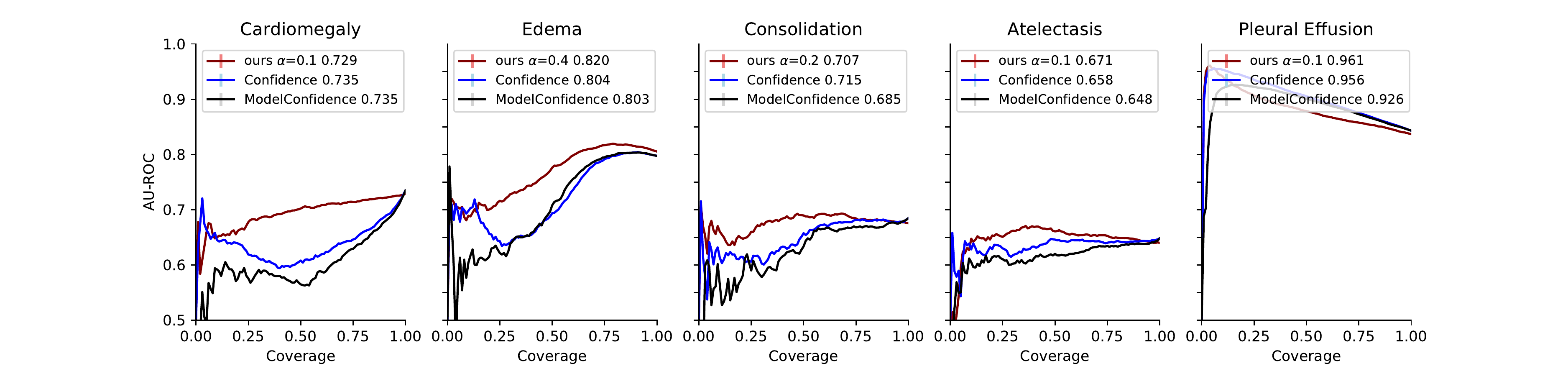}
\caption{classifier AU-ROC on non-deferred examples vs coverage for expert $q=0.7,p=1$ with 10\% of training data.}
      \label{fig:classauc_vs_cov_toy_10}
\end{subfigure}

\caption{Plot of classifier AU-ROC on non-deferred examples versus coverage for (a) for systems learned with 100\% of training data  (b) and learned with 10\% of training data. Noise at low coverage is due to reduced data size.}
\label{fig:plots_classauc_ap_toyexpert}
\end{figure}

\clearpage
\section{Deferred Proofs and Derivations}\label{apx:proofs}

\subsection{Section \ref{sec:surrog}}
\subsubsection{Binary Setting}
As we eluded to in the body of the paper, we can extend the losses introduced by \cite{cortes2016learning} to our setting for binary labels. 
Let $\mathcal{Y}=\{-1,+1\}$ and  $r,h: \mathcal{X} \to \mathbb{R}$ where we defer if $r(x)\leq 0$, for generality we assume $l_{exp}(x,y,m) = \max(c,\mathbb{I}_{m \neq y})$ as this allows to treat rejection learning as an immediate special case.
Following the derivation in \cite{cortes2016learning}, let $u \to \phi(-u)$ and $u \to \psi(-u)$ be two convex function upper bounding $\bI_{u\leq 0}$ and let $\alpha,\beta >0$, then:
\begin{flalign}
 &\nonumber L_c(h,r,x,y,m)  =\mathbb{I}_{h(x)y\leq 0} \bI_{r(x) > 0} + \max(c,\mathbb{I}_{m \neq y}) \bI_{r(x) \leq 0}&\\
 \nonumber&\leq\max\left\{\mathbb{I}_{\max\{h(x)y,-r(x)\}\leq 0} , \max(c,\mathbb{I}_{m \neq y}) \bI_{r(x) \leq 0}  \right\}\\
  \nonumber&\overset{(a)}{\leq}\max\left\{\mathbb{I}_{ \frac{\alpha}{2}(h(x)y-r(x))\leq 0} , \max(c,\mathbb{I}_{m \neq y}) \bI_{ \beta r(x) \leq 0}  \right\}\\
    &\overset{(b)}{\leq}\max\{\phi\left( \frac{-\alpha}{2}(h(x)y-r(x))\right) , \max(c,\mathbb{I}_{m \neq y}) \psi\left( -\beta r(x) \right)  \} \label{surogate:MH}\\
        &\leq\phi\left( \frac{-\alpha}{2}(h(x)y-r(x))\right) + \max(c,\mathbb{I}_{m \neq y}) \psi\left( -\beta r(x) \right) \label{surogate:PH}
\end{flalign}
step $(a)$ is by noting that $max(a,b)\geq \frac{a+b}{2}$, step $(b)$ since $\phi(u)$ and $\psi(u)$ upper bound $\bI_{u\leq0}$. Both the right hand sides of equations \eqref{surogate:MH} and \eqref{surogate:PH} are convex functions of both $h$ and $r$.
When $\phi$ and $\psi$ are both the exponential loss we obtain the following loss with $\beta(x,y,m): \mathcal{X} \times \mathcal{Y}^2 \to \mathbb{R}^+$:
\begin{align*}
     \nonumber &L_{SH}(h,r,x,y,m):= \exp\left( \frac{\alpha}{2}(r(x)-h(x)y)\right) + (c+\mathbb{I}_{m \neq y}) \exp\left( -\beta(x,y,m) r(x) \right)
\end{align*}
we will see that it will be necessary that $\beta$ is no longer constant for the loss to be consistent while in the standard case it sufficed to have $\beta$ constant \cite{cortes2016learning}. The following proposition shows that for an appropriate choice of $\beta$ and $\alpha$ we can make $L_{SH}$ consistent.

\begin{proposition}
 Let $c(x) = c - c\bP(Y \neq M|X=x) + \bP(Y \neq M|X=x)$, for $\alpha = 1$ and $\beta=\sqrt{\frac{1-c(x)}{c(x)}}$,  $\inf_{h,r}\bE_{x,y,m}[L_{SH}(h,r,x,y,m)]$ is attained at $(h^*_{SH},r^*_{SH})$ such that $sign(h^B)=sign(h^*_{SH})$ and $sign(r^B)=sign(r^*_{SH})$.
\end{proposition}
\begin{proof}
Denote $\eta(x)=\bP(Y=1|X=x)$ and $q(x,y)=\bP(M=1|X=x,Y=y)$, we have:
\begin{flalign*}
\inf_{h,r}\bE_{x,y,m}[L_{SH}(h,r,x,y,m)] &= \inf_{h,r} \bE_{x}\bE_{y|x}\bE_{m|x,y}[L_{SH}(h,r,x,y,m)] &\\
&= \bE_{x}\inf_{h(x),r(x)}\bE_{y|x}\bE_{m|x,y}[L_{SH}(h(x),r(x),x,y,m)]
\end{flalign*}
Now we will expand the inner expectation:
\begin{flalign}
&\label{prop:inner_loss_exp}\bE_{y|x}\bE_{m|x,y}[L_{SH}(h(x),r(x),x,y,m)]&\\\nonumber&= \eta(x) q(x,1) ( \exp\left( \frac{\alpha}{2}(r(x)-h(x))\right) + c \exp\left( -\beta r(x) \right) ) \\
&\nonumber+(1-\eta(x)) q(x,-1) ( \exp\left( \frac{\alpha}{2}(r(x)+h(x))\right) + (1) \exp\left( -\beta r(x) \right) )\\
&\nonumber+\eta(x) (1-q(x,1)) ( \exp\left( \frac{\alpha}{2}(r(x)-h(x))\right) + (1) \exp\left( -\beta r(x) \right) )\\
&\nonumber+(1-\eta(x)) (1-q(x,-1)) ( \exp\left( \frac{\alpha}{2}(r(x)+h(x))\right) + c \exp\left( -\beta r(x) \right) )
\end{flalign}
The Bayes optimal solution for our original loss in the binary setting is:
\begin{flalign*}
&h^B(x)= \eta(x) -\frac{1}{2} &\\
&r^B(x)= |\eta(x) -\frac{1}{2}| - (\frac{1}{2} -c - \bP(M \neq Y |X=x)) 
\end{flalign*}

\textbf{Case 1:} if $\eta(x)=0$, writing $v=r(x),u=h(x)$ then term \eqref{prop:inner_loss_exp} becomes:
\begin{equation*}
q(x,-1)( \exp\left( \frac{\alpha}{2}(v+u)\right) + 1 \exp\left( -\beta v \right) ) +(1-q(x,-1))( \exp\left( \frac{\alpha}{2}(v+u)\right) + c \exp\left( -\beta v \right) )
\end{equation*}
then to minimize the above it is necessary that the optimal solutions are such that $u^*<0,v^*>0$ which agree with the sign of the original Bayes solution.

\textbf{Case 2:} if $\eta(x)=1$, then term \eqref{prop:inner_loss_exp} becomes:
\begin{equation*}
q(x,1)( \exp\left( \frac{\alpha}{2}(v-u)\right) + c \exp\left( -\beta v \right) ) +(1-q(x,1))( \exp\left( \frac{\alpha}{2}(v-u)\right) + (1) \exp\left( -\beta v \right) )
\end{equation*}
then to minimize the above it is necessary that the optimal solutions are such that $u^*>0,v^*>0$ which agree with the sign of the original Bayes solution.

\textbf{Case 3:} $\eta(x)\in(0,1)$, for ease of notation denote the RHS of equation \eqref{prop:inner_loss_exp} as $L_{\psi}(u,v)$, note that $L_{\psi}(u,v)$ is a convex function of both $u$ and $v$ and therefore to find the optimal solution it suffices to take the partial derivatives with respect to each and set them to $0$.

For $u$:
\begin{flalign*}
&\frac{\partial_{\psi}(u,v)}{\partial u} = 0 &\\
&\iff - \eta(x) \frac{\alpha}{2} \exp\left( \frac{\alpha}{2}(v-u^*)\right) + (1-\eta(x))\exp\left( \frac{\alpha}{2}(v+u^*)\right) = 0 \\
&\iff - \eta(x) \frac{\alpha}{2} \exp\left( \frac{-\alpha}{2}u^*\right) + (1-\eta(x)) \frac{\alpha}{2}\exp\left( \frac{\alpha}{2}u^*\right) = 0 \\
&\iff u^* = \frac{1}{\alpha} \log(\frac{\eta(x)}{1- \eta(x)})
\end{flalign*}
we note that $u^*$ has the same sign as the minimizer of the exponential loss and hence has the same sign as $h^B(x)$.

Plugging $u^*$ and taking the derivative with respect to $v$:

\begin{flalign*}
&\frac{\partial_{\psi}(u^*,v)}{\partial v} = 0 &\\
&\iff  \eta(x) \frac{\alpha}{2} \exp\left( \frac{\alpha}{2}(v^*-u^*)\right) + (1-\eta(x))\exp\left( \frac{\alpha}{2}(v^*+u^*)\right)\\ &- \beta c(\eta(x)q(x,1) + (1-\eta(x))(1-q(x,-1)) \exp(-\beta v^*) \\&-(1-\eta(x))q(x,-1)  \beta  \exp(-\beta v^*)  -\eta(x)(1-q(x,1))  \beta  \exp(-\beta v^*) = 0   \\
&\iff  \eta(x) \frac{\alpha}{2} \exp\left( \frac{\alpha}{2}(v^*-u^*)\right) + (1-\eta(x))\exp\left( \frac{\alpha}{2}(v^*+u^*)\right) \\&- \beta (c -c\bP(M\neq Y|X=x) + \bP(M\neq Y|X=x) ) \exp(-\beta v^*)  =0 
\end{flalign*}
Appealing to the proof of Theorem 1 in \cite{cortes2016boosting} we obtain that:
\begin{equation*}
    v^* = \frac{1}{\alpha/2 + \beta} \log \left( \frac{c(x)\beta}{\alpha} \sqrt{\frac{1}{\eta(x)(1-\eta(x))}} \right)
\end{equation*}

Furthermore by the proof of Theorem 1 in \cite{cortes2016boosting}, the sign of $v^*$ matches that of $r^B(x)$ if and only if:
\[
\frac{\beta}{\alpha} = \sqrt{\frac{1-c(x)}{c(x)}}
\]

\end{proof}

\subsubsection{Multiclass setting}

\textbf{Proposition 1.} \textit{ \noindent
 $\tilde{L}_{CE}$  is  convex and is a consistent loss function for $\tilde{L}$:
 \begin{center}
      let $\bm{\tilde{g}}=\arg  \inf_{\mathbf{g}} \bE\left[ \tilde{L}_{CE}(\mathbf{g},\mathbf{c}) |X=x\right]$, then:  $\arg  \max_{i \in [K+1]}\bm{\tilde{g}}_i = \arg \min_{i \in [K+1]} \bE[c(i)|X=x]$
 \end{center}
 }
\begin{proof}
Writing the expected loss:
\begin{flalign*}
\inf_{\mathbf{g}}\bE_{x,\mathbf{c}}[\tilde{L}_{CE}(\mathbf{g},x,\mathbf{c})] &= \inf_{\mathbf{g}} \bE_{x}\bE_{\mathbf{c}|x}[\tilde{L}_{CE}(\mathbf{g},x,\mathbf{c})] = \bE_{x}\inf_{\mathbf{g}(x)}\bE_{\mathbf{c}|x}[\tilde{L}_{CE}(\mathbf{g}(x),x,\mathbf{c})]
\end{flalign*}
Now we will expand the inner expectation:
\begin{flalign}
&\bE_{\mathbf{c}|x}[\tilde{L}_{CE}(\mathbf{g}(x),x,\mathbf{c})] \nonumber = - \sum_{y \in [K+1]} \bE[\max_j c(j) - c(y)|X=x] \log\left( \frac{\exp(g_y(x))}{\sum_k \exp(g_k(x))} \right)
\end{flalign}
The loss $\tilde{L}_{CE}$ is convex in the predictor, so it suffices to differentiate with respect to each $g_y$ for $y \in \mathcal{Y}^\bot$ and set to 0. 

\begin{flalign*}
&\nonumber\frac{\partial L_{CE}}{\partial g_y^*} = 0 &\\ 
&\nonumber\iff    \bE[\max_j c(j) - c(y)|X=x] -   \frac{\exp(g_y^*(x))}{\sum_k \exp(g_k(x))}   \sum_{i \in  [K+1]} \bE[\max_j c(j) - c(i)|X=x]  = 0 \\
& \iff   \frac{\exp(g_y^*(x))}{\sum_k \exp(g_k(x))}  = \frac{ \bE[\max_j c(j) - c(y)|X=x]}{\sum_{i \in  [K+1]} \bE[\max_j c(j) - c(i)|X=x]}
\end{flalign*}
From this we can deduce:
\begin{flalign*}
h(x) &= \arg\max_{y \in  [K+1]} g_y^*(x)= \arg \max_{y \in  [K+1]} \frac{\exp(g_{y}^*(x))}{\sum_{y \in  [K+1]}\exp(g_{y}^*(x))}&\\& = \arg \max_{y \in  [K+1]} \frac{ \bE[\max_j c(j) |X=x] -\bE[ c(y)|X=x]}{\sum_{i \in  [K+1]} \bE[\max_j c(j) - c(i)|X=x]} \\
&= \arg \min_{y \in  [K+1]}  \bE[ c(y)|X=x] =\tilde{h}^B(x)
\end{flalign*}
\end{proof}

\textbf{Proposition 2.} \textit{ \noindent  The minimizers of the loss $L_{0{-}1}$ \eqref{eq:01_reject_loss} are defined point-wise for all $x\in \mathcal{X}$ as:
\begin{align}
    &h^B(x) = \arg \max_{y \in \mathcal{Y}}\eta_y(x)  \nonumber    \\  
    &r^B(x)= \bI_ {\max_{y \in \mathcal{Y}}\eta_y(x) \leq  \bP(Y = M|X=x) }  
\end{align}
}\begin{proof}
When we don't defer, the loss incurred by the model is the misclassification loss in the standard multiclass setting and hence by standard arguments \cite{friedman2001elements} we can define $h^B$ point-wise regardless of  $r$:
\begin{flalign*}
    h^B(x) = \arg \inf_{h} \bE_{y} [\bI_{h \neq y}] =  \arg \max_{y \in \mathcal{Y}}\eta_y(x)
\end{flalign*}
Now for the rejector, we should only defer if the expected loss of having the expert predict is less than the error of the classifier $h^B$ defined above, define $r^B: \mathcal{X} \to \{0,+1 \}$ as:
\begin{flalign*}
    r^B(x) &= \bI_{\bE[\bI_{M \neq Y}|X=x]  \leq \bE[\bI_{h^B(x) \neq Y}|X=x] }\\
    &=\bI_{ \bP(Y \neq M)  \leq (1-\max_{y \in \mathcal{Y}}\eta_y(x)) } \\
    &= \bI_{\bP(Y = M)  \geq \max_{y \in \mathcal{Y}}\eta_y(x) }
\end{flalign*}
\end{proof}

\textbf{Theorem 2.} \textit{ \noindent
The loss $L_{CE}$ is a convex upper bound of $L_{0{-}1}$ and is consistent:\\
$\inf_{h,r}\bE_{x,y,m}[L_{CE}(h,r,x,y,m)]$ is attained at $(h^*_{CE},r^*_{CE})$ such that $h^B(x)=h^*_{CE}(x)$ and $r^B(x)=r^*_{CE}(x)$ for all $x \in \mathcal{X}$.
}
\begin{proof}
The fact that $L_{CE}$ is convex is immediate as $\bI_{m = y}\geq 0$ and the cross entropy loss is convex.

Now we show that $L_{CE}$ is an upper bound of $L_{0{-}1}$:
\begin{align}
 \nonumber & L_{0{-}1}(h,r,x,y,m)= \mathbb{I}_{h(x)\neq y } \bI_{r(x) = 0} + \bI_{m \neq y} \bI_{r(x) = 1}&\\
 &\overset{(a)}{\leq}-  \log\left(\frac{\exp(g_{y}(x))}{\sum_{y' \in \mathcal{Y} \cup \bot}\exp(g_{y'}(x))} \right) - \bI_{m= y} \log\left(\frac{\exp(g_{\bot}(x))}{\sum_{y' \in \mathcal{Y} \cup \bot}\exp(g_{y'}(x))} \right) \label{eq:multiclass_entropy_loss}
\end{align}
To justify inequality $(a)$, consider first if $r(x)=0$, then if $\mathbb{I}_{h(x)\neq y }=1$ we know that $\frac{\exp(g_{y}(x))}{\sum_{y' \in \mathcal{Y} \cup \bot}\exp(g_{y'}(x))} \leq \frac{1}{2}$ giving $-\log\left(\frac{\exp(g_{y}(x))}{\sum_{y' \in \mathcal{Y} \cup \bot}\exp(g_{y'}(x))} \right)  \geq 1$, moreover all the terms in the RHS of $(a)$ are always positive.

On the other hand if $r(x)=1$, then again  $\frac{\exp(g_{y}(x))}{\sum_{y' \in \mathcal{Y} \cup \bot}\exp(g_{y'}(x))} \leq \frac{1}{2}$ as we decided to reject and since also giving $-\log\left(\frac{\exp(g_{y}(x))}{\sum_{y' \in \mathcal{Y} \cup \bot}\exp(g_{y'}(x))} \right)  \geq 1$. Finally note that $ L_{0{-}1}(h,r,x,y,m)\leq 1$.

 We will now show that the optimal rejector minimizing the upper bound \eqref{eq:multiclass_entropy_loss} is in fact consistent.



Denote $q_m(x,y)=\bP(M=m|X=x,Y=y)$ and $\eta_y(x) = \bP(Y=y|X=x)$, we have:
\begin{flalign*}
\inf_{h,r}\bE_{x,y,m}[L_{CE}(h,r,x,y,m)] &= \inf_{h,r} \bE_{x}\bE_{y|x}\bE_{m|x,y}[L_{CE}(h,r,x,y,m)] 
\\&= \bE_{x}\inf_{h(x),r(x)}\bE_{y|x}\bE_{m|x,y}[L_{CE}(h(x),r(x),x,y,m)]
\end{flalign*}
Let us expand the inner expectation:
\begin{flalign}
\nonumber&\bE_{y|x}\bE_{m|x,y}[L_{CE}(h(x),r(x),x,y,m)]&\\\nonumber &= \bE_{y|x} \left[ -\log\left(\frac{\exp(g_{y}(x))}{\sum_{y' \in \mathcal{Y} \cup \bot}\exp(g_{y'}(x))} \right)- \sum_{m \in \mathcal{Y}} \bI_{m = y} \nonumber \log\left(\frac{\exp(g_{\bot}(x))}{\sum_{y' \in \mathcal{Y} \cup \bot}\exp(g_{y'}(x))} \right) \right] \\\nonumber
&=- \sum_{y \in \mathcal{Y}} \eta_y(x) \log\left(\frac{\exp(g_{y}(x))}{\sum_{y' \in \mathcal{Y} \cup \bot}\exp(g_{y'}(x))} \right) \\&- \sum_{y \in \mathcal{Y}} \eta_y(x) \sum_{m \in \mathcal{Y}} q_m(x,y)\bI_{m = y} \nonumber \log\left(\frac{\exp(g_{\bot}(x))}{\sum_{y' \in \mathcal{Y} \cup \bot}\exp(g_{y'}(x))} \right) \\\nonumber
&\overset{(a)}{=}- \sum_{y \in \mathcal{Y}} \eta_y(x) \log\left(\frac{\exp(g_{y}(x))}{\sum_{y' \in \mathcal{Y} \cup \bot}\exp(g_{y'}(x))} \right)  - \sum_{y \in \mathcal{Y}} \eta_y(x) q_y(m,y) \log\left(\frac{\exp(g_{\bot}(x))}{\sum_{y' \in \mathcal{Y} \cup \bot}\exp(g_{y'}(x))} \right) \\ 
&\overset{(b)}{=}- \sum_{y \in \mathcal{Y}} \eta_y(x) \log\left(\frac{\exp(g_{y}(x))}{\sum_{y' \in \mathcal{Y} \cup \bot}\exp(g_{y'}(x))} \right)  \nonumber\\&- \bP(Y=M |X=x) \log\left(\frac{\exp(g_{\bot}(x))}{\sum_{y' \in \mathcal{Y} \cup \bot}\exp(g_{y'}(x))} \right) \label{prop:y|x_cross_loss_withq}
\end{flalign}
In step $(a)$ all terms that differed on $y$ and $m$ disappear, in step $(b)$ we have:
\begin{flalign*}
&\sum_{y \in \mathcal{Y}} \eta_y(x) q_y(m,y) = \sum_{y \in \mathcal{Y}} \bP(M=y,Y=y|X=x)
= \bP(Y=M|X=x)
\end{flalign*}

 For ease of notation denote the RHS of equation \eqref{prop:y|x_cross_loss_withq} as $L_{CE}(g_1,\cdots,g_{|\mathcal{Y}|},g_\bot)$, note that it is a a convex function, hence we will take the partial derivatives with respect to each argument and set them to $0$.

For any $g_\bot$, and for $i \in \mathcal{Y}$ we have :
\begin{flalign}
&\nonumber\frac{\partial L_{CE}(g_1^*,\cdots,g_{|\mathcal{Y}|^*},g_\bot)}{\partial g_i^*} = 0 &\\ 
&\label{prop:optimal_g}\iff   \frac{\exp(g_{i}^*(x))}{\sum_{y' \in \tilde{\mathcal{Y}}}\exp(g_{y'}^*(x))} = \frac{\eta_i(x)}{1 + \bP(Y=M|X=x)}
\end{flalign}
The optimal $h^*$ for any $g_\bot$ should satisfy equation \eqref{prop:optimal_g} for every $i \in \mathcal{Y}$, however
since exponential is an increasing function we get that the optimal $h^*$ in fact agrees with the Bayes solution as:
\begin{flalign*}
\arg\max_{y \in \mathcal{Y}} g_y^*(X)&= \arg \max_{y \in \mathcal{Y}} \frac{\exp(g_{y}^*(x))}{\sum_{y \in \mathcal{Y}}\exp(g_{y}^*(x)) + \exp(g_\bot(x))}&\\& = \arg \max_{y \in \mathcal{Y}} \frac{\eta_y(x)}{1 +  \bP(Y=M|X=x)} = h^B(x)
\end{flalign*}

Plugging $h^*$ and taking the derivative with respect to the optimal $g_\bot^*$:

\begin{flalign}
&\nonumber\frac{\partial L_{CE}(g_1^*,\cdots,g_{|\mathcal{Y}|}^*,g_\bot^*)}{\partial g_\bot^*} = 0 &\\ 
& \nonumber \iff  \frac{\exp(g_{\bot}^*(x))}{\sum_{y' \in \mathcal{Y}}\exp(g_{y'}^*(x))} = \frac{ \bP(Y=M|X=x)}{1+ \bP(Y=M|X=x)}
\end{flalign}
Note note that $r^*(x)=1$ only if $\bP(Y=M|X=x) \geq \max_{y \in \mathcal{Y}} \eta_y(x) $ which agrees with $r^B(x)$
\end{proof}

\subsection{Section \ref{sec:theory}}

\textbf{Proposition 3.}\textit{ \noindent 
$L_{mix}$ is realizable $(\mathcal{H},\mathcal{R})$-consistent for classes closed under scaling but is not  classification consistent.
}
\begin{proof}

We first prove that $L_{mix}$ is realizable $(\mathcal{H},\mathcal{R})$-consistent.
Let $\mathbf{P}$ and $M$ be such that  there exists $h^*,r^* \in \mathcal{H} \times \mathcal{R}$ that have zero error $L(h^*,r^*)=0$. Assume  that  $(\hat{h},\hat{r})$ satisfy 
\[\left| \bE[L_{mix}(\hat{h},\hat{r},x,y,m)] - \inf_{h \in \cH, r \in \cR}  \bE[L_{mix}(h,r,x,y,m)] \right| \leq \delta\]

Let $u>0$, we have:
\begin{align}
\nonumber
&\bE[L(\hat{h},\hat{r},x,y,m)] \\&\leq \nonumber 2\bE[L_{mix}(\hat{h},\hat{r},x,y,m)] \quad \textrm{ (factor of 2 is upper bound)}\\ \nonumber
&\leq 2\bE[L_{mix}(uh^*,ur^*,x,y,m)]+ 2\delta \quad \textrm{(by assumption and closed under scaling) }\\ \nonumber
&= 2\bE[L_{mix}(uh^*,ur^*,x,y,m)|r^*=1]\bP(r^*=1) + 2\bE[L_{mix}(uh^*,ur^*,x,y,m)|r^*=0]\bP(r^*=0) + 2\delta \quad \textrm{ }\\ \nonumber
&= 2\bE[-\log\left(\frac{\exp(ug_{y}(x))}{\sum_{y' \in \mathcal{Y}}\exp(ug_{y'}(x))} \right) \frac{\exp(ur_{0}(x))}{\sum_{i \in \{0,1\}}\exp(ur_{i}(x))}  + \bI_{m \neq y} \frac{\exp(ur_{1}(x))}{\sum_{i \in \{0,1\}}\exp(ur_{i}(x))}|r^*=0]\bP(r^*=0) \\&+  2\bE[-\log\left(\frac{\exp(ug_{y}(x))}{\sum_{y' \in \mathcal{Y}}\exp(ug_{y'}(x))} \right) \frac{\exp(ur_{0}(x))}{\sum_{i \in \{0,1\}}\exp(ur_{i}(x))} \nonumber \\ &+ \bI_{m \neq y} \frac{\exp(ur_{1}(x))}{\sum_{i \in \{0,1\}}\exp(ur_{i}(x))}|r^*=1]\bP(r^*=1) + 2\delta \quad \textrm{ } \label{eq:proof_mixexp_realizable}
\end{align} 
Let us examine each term in the RHS of \eqref{eq:proof_mixexp_realizable}, when $r^*=1$ we have $r_1(x)>r_0(x)$ hence: \[\lim_{u \to \infty}\frac{\exp(ur_{0}(x))}{\sum_{i \in \{0,1\}}\exp(ur_{i}(x))} = 0 \]
Furthermore it most be that $\bI_{m \neq y} = 0$ as we decided to defer. 

When $r^*=0$, we have $r_0(x) \geq r_1(x)$ hence: \[\lim_{u \to \infty}\frac{\exp(ur_{1}(x))}{\sum_{i \in \{0,1\}}\exp(ur_{i}(x))} = 0 \]
moreover we have $h^*(x)=y$ by optimality of $(h^*,r^*)$ (as we did not defer) and realizability thus:
\[
\lim_{u \to \infty} \log\left(\frac{\exp(ug_{y}(x))}{\sum_{y' \in \mathcal{Y}}\exp(ug_{y'}(x))} \right) = 0
\]

We can conclude that
taking the limit as $u \to \infty$ on the RHS of \eqref{eq:proof_mixexp_realizable} and applying the monotone convergence theorem (swap of expectation and limit) we get:
\begin{align*}
&\bE[L(\hat{h},\hat{r},x,y,m)] \leq 2 \delta
\end{align*} 
taking $\delta = \epsilon/2$ completes the proof.

We now move to looking at the Bayes solution of $L_{mix}$,
denote $q_m(x,y)=\bP(M=m|X=x,Y=y)$, we have:
\begin{flalign*}
\inf_{h,r}\bE_{x,y,m}[L_{mix}(h,r,x,y,m)] 
= \bE_{x}\inf_{h(x),r(x)}\bE_{y|x}\bE_{m|x,y}[L_{mix}(h(x),r(x),x,y,m)]
\end{flalign*}
Let us expand the inner expectation:
\begin{flalign}
\label{eq:proof_mixofexp_notconsistent}&\bE_{y|x}\bE_{m|x,y}[L_{mix}(h(x),r(x),x,y,m)]=& \\
& - \sum_{y \in \mathcal{Y}} \eta_y(x)  \log\left(\frac{\exp(g_{y}(x))}{\sum_{y' \in \mathcal{Y} }\exp(g_{y'}(x))} \right) \frac{\exp(r_{0}(x))}{\sum_{i \in \{0,1\}}\exp(r_{i}(x))} + \bP(Y \neq M |X=x) \frac{\exp(r_{1}(x))}{\sum_{i \in \{0,1\}}\exp(r_{i}(x))}  \nonumber
\end{flalign}

Denote the RHS of \eqref{eq:proof_mixofexp_notconsistent} by $L_{mix}(g_1,\cdots, g_{|\mathcal{Y}|},r_0,r_1)$, it is a convex function in $g_i$ for all $i \in \mathcal{Y}$, consider any $r_0,r_1$, we have :
\begin{flalign}
&\frac{\partial L_{mix}(g_1^*,\cdots,g_{|\mathcal{Y}|^*},r_0,r_1)}{\partial g_i^*} = 0 
\label{prop:optimal_g_mix_of_exp}\iff   \frac{\exp(g_{i}^*(x))}{\sum_{y' \in \mathcal{Y}}\exp(g_{y'}^*(x))} = \eta_i(x)
\end{flalign}
Since the optimal $h^*$ for any $r_0,r_1$  \emph{does not depend} on the form of $r_0$ and $r_1$ we conclude that \eqref{prop:optimal_g_mix_of_exp} gives the optimal choice of $h$. We now need to find the optimal choice of $r_0(x)$ and $r_1(x)$ to minimize $L_{mix}(g_1^*,\cdots,g_{|\mathcal{Y}|^*},r_0,r_1)$ which takes the following form:

\begin{equation*}
  L_{mix}(g_1^*,\cdots,g_{|\mathcal{Y}|^*},r_0,r_1) =   \textrm{H}(h^B(x)) \frac{\exp(r_{0}(x))}{\sum_{i \in \{0,1\}}\exp(r_{i}(x))} + \bP(Y \neq M |X=x) \frac{\exp(r_{1}(x))}{\sum_{i \in \{0,1\}}\exp(r_{i}(x))}
\end{equation*}
where $\textrm{H}(X)$ is the Shannon entropy of the random variable $X$, here by $\textrm{H}(h^B(x))$ we refer to the entropy of the probabilistic form of $h^B(x)$ according to \eqref{prop:optimal_g_mix_of_exp} . Clearly the optimal $r_0^*$ and $r_1^*$ have the following behavior for a given $x \in \mathcal{X}$:
\[
\begin{cases}
r_0(x)= \infty, r_1(x) = - \infty \quad if  \ \textrm{H}(h^B(x)) <  \bP(Y \neq M |X=x) \\
r_0(x)= - \infty, r_1(x) =  \infty \quad if \   \textrm{H}(h^B(x)) \geq  \bP(Y \neq M |X=x) 
\end{cases}
\]
This does not have the form of $r^B(x)$, as this rejector compares the entropy of $h^B(x)$ instead of it's confidence to the probability of error of the expert which will not always be in accordance.
\end{proof}

\textbf{Theorem 2.}\textit{ \noindent For any expert $M$ and data distribution $\mathbf{P}$ over $\mathcal{X} \times \mathcal{Y}$, let $0<\delta<\frac{1}{2}$, then  with probability at least $1-\delta$, the following holds for the empirical minimizers $(\hat{h}^*,\hat{r}^*)$:
\begin{align}
    L_{0{-}1}(\hat{h}^*,\hat{r}^*) &\leq  L_{0{-}1}(h^*,r^*) + \mathfrak{R}_n(\mathcal{H}) +  \mathfrak{R}_{n}(\mathcal{R})  + \mathfrak{R}_{n \bP(M \neq Y)/2}(\mathcal{R})  \nonumber \\
    & + 2\sqrt{\frac{\log{(\frac{2}{\delta})}}{2n}} +\frac{\bP(M\neq Y)}{2}  \exp\left(- \frac{n \bP(M \neq Y)}{8} \right)  \nonumber
\end{align}
}
\begin{proof}
Let $\mathcal{L}_{\mathcal{H},\mathcal{R}}$ be the family of functions defined as $\mathcal{L}_{\mathcal{H},\mathcal{R}}=\{(x,y,m) \to L(h,r,x,y,m); h \in \mathcal{H}, r \in \mathcal{R}  \}$ with $ L(h,r,x,y,m):= \mathbb{I}_{h(x)\neq y} \bI_{r(x) = -1} + \mathbb{I}_{m \neq y} \bI_{r(x) = 1}$. Let $\mathfrak{R}_n(\mathcal{L}_{\mathcal{H},\mathcal{R}})$ be the Rademacher complexity of $\mathcal{L}_{\mathcal{H},\mathcal{R}}$, then since $L(h,r,x,y,m) \in [0,1]$, by the standard Rademacher complexity bound (Theorem 3.3 in \cite{mohri2018foundations}), with probability at least $1-\delta/2$ we have:
\begin{equation*}
    L_{0{-}1}(\hat{h}^*,\hat{r}^*) \leq L^S_{0{-}1}(\hat{h}^*,\hat{r}^*) + 2 \mathfrak{R}_n(\mathcal{L}_{\mathcal{H},\mathcal{R}}) + \sqrt{\frac{\log{(\frac{2}{\delta})}}{2n}}
\end{equation*}
We will now relate the complexity of $\mathcal{L}_{\mathcal{H},\mathcal{R}}$ to the individual classes:
\begin{flalign}
\nonumber&\mathfrak{R}_n(\mathcal{L}_{\mathcal{H},\mathcal{R}})=\nonumber \bE_{\boldsymbol{\epsilon}}[ \sup_{(h,r)\in \mathcal{H} \times \mathcal{R}} \frac{1}{m} \sum_{i=1}^m \epsilon_i    \mathbb{I}_{h(x_i)\neq y_i} \bI_{r(x_i) = -1} + \epsilon_i\mathbb{I}_{m_i \neq y_i} \bI_{r(x_i) = 1} ]& \nonumber\\
 \nonumber&\overset{(a)}{\leq} \bE_{\boldsymbol{\epsilon}}\left[ \sup_{(h,r)\in \mathcal{H} \times \mathcal{R}} \frac{1}{m} \sum_{i=1}^m \epsilon_i    \mathbb{I}_{h(x_i)\neq y_i} \bI_{r(x_i) = -1}\right]  \\\nonumber&+  \bE_{\boldsymbol{\epsilon}}\left[ \sup_{(h,r)\in \mathcal{H} \times \mathcal{R}} \frac{1}{m} \sum_{i=1}^m  \epsilon_i\mathbb{I}_{m_i \neq y_i} \bI_{r(x_i) = 1} \right] \\
 \nonumber&\overset{(b)}{\leq} \bE_{\boldsymbol{\epsilon}}\left[ \sup_{(h,r)\in \mathcal{H} \times \mathcal{R}} \frac{1}{m} \sum_{i=1}^m \epsilon_i    \mathbb{I}_{h(x_i)\neq y_i} \right] +\nonumber \bE_{\boldsymbol{\epsilon}}\left[ \sup_{(h,r)\in \mathcal{H} \times \mathcal{R}} \frac{1}{m} \sum_{i=1}^m  \epsilon_i \bI_{r(x_i) = -1} \right] \\& +  \bE_{\boldsymbol{\epsilon}}\left[ \sup_{(h,r)\in \mathcal{H} \times \mathcal{R}} \frac{1}{m} \sum_{i=1}^m  \epsilon_i\mathbb{I}_{m_i \neq y_i} \bI_{r(x_i) = 1} \right] \nonumber \\
&\leq \frac{1}{2} \mathfrak{R}_n(\mathcal{H}) + \frac{1}{2} \mathfrak{R}_n(\mathcal{R}) +\bE_{\boldsymbol{\epsilon}}\left[ \sup_{(h,r)\in \mathcal{H} \times \mathcal{R}} \frac{1}{m} \sum_{i=1}^m  \epsilon_i\mathbb{I}_{m_i \neq y_i} \bI_{r(x_i) = 1} \right] \label{th1:exp_term}
\end{flalign} 

step $(a)$ follows as the supremum is a subadditive function , step $(b)$ is the application of Lemma 2 in \cite{desalvo2015learning} to $\bE_{\boldsymbol{\epsilon}}\left[ \sup_{(h,r)\in \mathcal{H} \times \mathcal{R}} \frac{1}{m} \sum_{i=1}^m \epsilon_i    \mathbb{I}_{h(x_i)\neq y_i} \bI_{r(x_i) = -1}\right]$ which says that the Rademacher complexity of a product of two indicators functions is upper bounded by the sum of the complexities of each class, now we will take a closer look at the last term in the RHS of inequality \eqref{th1:exp_term}.
Denote $n_m^S = \sum_{j \in S} \bI_{y_j\neq m_j} $ and define the random variable $S_m = \{i: y_i \neq m_i \}$, we have that $n_m^S \sim \textrm{Binomial}(n,\bP(M \neq Y))$ and $\bE[n_m^S|S_m]=n \bP(M\neq Y)$, hence:
\begin{flalign*}
    &\bE\left[ \sup_{(h,r)\in \mathcal{H} \times \mathcal{R}} \frac{1}{m} \sum_{i=1}^m  \epsilon_i\mathbb{I}_{m_i \neq y_i} \bI_{r(x_i) = 1} \right] 
    \\&=\bE\left[ \sup_{(h,r)\in \mathcal{H} \times \mathcal{R}} \frac{1}{m} \sum_{i=1 \  s.t. \ y_i \neq m_i}^m  \epsilon_i \bI_{r(x_i) = 1} \right] &\\
    &=  \bE\left[ \frac{n_m^S}{m}\sup_{(h,r)\in \mathcal{H} \times \mathcal{R}} \frac{1}{n_m^S} \sum_{i=1 }^{n_m^S}  \epsilon_i \bI_{r(x_i) = 1} \right] \textrm{(by relabeling)} \\
    &\overset{(a)}{=}  \bE\left[ \bE_{\boldsymbol{\epsilon}} \left[\frac{n_m^S}{m}\sup_{(h,r)\in \mathcal{H} \times \mathcal{R}} \frac{1}{n_m^S} \sum_{i=1 }^{n_m^S}  \epsilon_i \bI_{r(x_i) = 1}|S_m \right]  \right]   \\
    &\overset{(b)}{=}  \bE\left[ \frac{n_m^S}{m}  \hat{\mathfrak{R}}_{S_m}(\mathcal{R})   \right]  \\
    & \overset{(c)}{=} \bP(n_m^S < \frac{n \bP(A)}{2})  \bE\left[ \frac{n_m^S}{m}  \hat{\mathfrak{R}}_{S_m}(\mathcal{R})|n_m^S < \frac{n \bP(A)}{2}   \right]    + \bP(n_m^S \geq \frac{n \bP(A)}{2})  \bE\left[ \frac{n_m^S}{m}  \hat{\mathfrak{R}}_{S_m}(\mathcal{R})|n_m^S \geq \frac{n \bP(A)}{2}   \right] \\
    & \overset{(d)}{\leq} \frac{\bP(M\neq Y)}{2} \exp\left(- \frac{n \bP(M \neq Y)}{8} \right)+ \mathfrak{R}_{n \bP(M \neq Y)/2}(\mathcal{R})
\end{flalign*}
In step $(a)$ we conditioned on the dataset $S_m$, in step $(b)$ we used the definition of the empirical Rademacher complexity $ \hat{\mathfrak{R}}_{S_m}(\mathcal{R})$  on $S_m$, step $(c)$ we introduce the event $A= \{ M \neq Y\}$, step $(d)$ follows from a Chernoff bound on $n_m^S$ and since the Rademacher complexity is bounded by $1$ and is non-increasing with respect to sample size.

We can now proceed with inequality \eqref{th1:exp_term}:
\begin{flalign*}
&\mathfrak{R}_n(\mathcal{L}_{\mathcal{H},\mathcal{R}})
\overset{(a)}{\leq} \frac{1}{2} \mathfrak{R}_n(\mathcal{H}) + \frac{1}{2} \mathfrak{R}_n(\mathcal{R}) + \frac{\bP(M\neq Y)}{2}  \exp\left(- \frac{n \bP(M \neq Y)}{8} \right) + \mathfrak{R}_{n \bP(M \neq Y)/2}(\mathcal{R}) & \nonumber \\
\end{flalign*} 
step $(a)$ follows as the Rademacher complexity of indicator functions based on a certain class is equal to half the Rademacher complexity of the class \cite{mohri2018foundations}.

The final step is to note by Hoeffding's inequality we have with probability at least $1- \delta/2$:
\begin{flalign*}
   & L^S(h^*,r^*)   \leq L(h^*,r^*) + \sqrt{\frac{\log{(\frac{2}{\delta})}}{2n}}&
\end{flalign*}

Now since $(\hat{h}^*,\hat{h}^*)$ are the empirical minimizers we have that $ L^S(\hat{h}^*,\hat{r}^*) \leq  L^S(h^*,r^*)$, collecting all the inequalities we obtain the following generalization bound with probability at least $1- \delta$:

\begin{flalign}
    L(\hat{h}^*,\hat{r}^*) &\leq  L(h^*,r^*) + \mathfrak{R}_n(\mathcal{H}) \nonumber + \mathfrak{R}_{n}(\mathcal{R}) + 2\sqrt{\frac{\log{(\frac{2}{\delta})}}{2n}} \nonumber &\\& + \frac{\bP(M\neq Y)}{2}  \exp\left(- \frac{n \bP(M \neq Y)}{8} \right) + \mathfrak{R}_{n \bP(M \neq Y)/2}(\mathcal{R}) \nonumber
\end{flalign}
\end{proof}

\end{document}